\documentclass{article}

\PassOptionsToPackage{authoryear,round}{natbib}
\usepackage{paper_style}

\usepackage[utf8]{inputenc}
\usepackage[T1]{fontenc}
\usepackage{amsmath,amssymb,amsthm,mathtools}
\usepackage{amsfonts}
\usepackage{bbm}
\usepackage{algorithm}
\usepackage[noend]{algpseudocode}
\usepackage{booktabs}
\usepackage{enumitem}
\usepackage{xcolor}
\usepackage[colorlinks=true,linkcolor=blue,citecolor=blue,urlcolor=blue,hyperfootnotes=false]{hyperref}
\usepackage[nameinlink,noabbrev]{cleveref}
\makeatletter
\renewcommand{\theHALG@line}{\thealgorithm.\arabic{ALG@line}}
\makeatother

\theoremstyle{plain}
\newtheorem{theorem}{Theorem}
\newtheorem{lemma}{Lemma}
\newtheorem{proposition}{Proposition}
\newtheorem{corollary}{Corollary}

\theoremstyle{definition}
\newtheorem{remark}{Remark}
\newtheorem{assumption}{Assumption}

\crefname{assumption}{Assumption}{Assumptions}
\Crefname{assumption}{Assumption}{Assumptions}
\crefname{lemma}{Lemma}{Lemmas}
\Crefname{lemma}{Lemma}{Lemmas}

\newcommand{\E}{\mathbb{E}}
\newcommand{\Var}{\mathrm{Var}}
\newcommand{\Prob}{\mathbb{P}}
\newcommand{\R}{\mathbb{R}}
\newcommand{\1}{\mathbbm{1}}

\newcommand{\psihat}{\widehat{\psi}_{\mathrm{cal}}}
\newcommand{\mhat}{m}

\newcommand{\rhozero}{\rho_0}
\newcommand{\rhon}{\rho_n}
\newcommand{\mtilde}{m_n^\star}
\newcommand{\muzero}{\mu_0}
\newcommand{\Pl}{\mathbb{P}^{L}_n}
\newcommand{\Pu}{\mathbb{P}^{U}_N}
\newcommand{\Px}{P_{0,X}}

\newcommand{\Cov}{\operatorname{Cov}}

\usepackage{fvextra}

\DefineVerbatimEnvironment{codeblock}{Verbatim}{
  breaklines=true,
  breakanywhere=true,
  fontsize=\small
}

\title{``Calibeating'' Prediction-Powered Inference}

\author{
Lars van der Laan\\
Department of Statistics, University of Washington\\[0.75em]
Mark van der Laan\thanks{Targeted ML Solutions, \href{https://www.targetedml.com/}{targetedml.com}}\\
Department of Statistics and Division of Biostatistics,\\
Center for Targeted Machine Learning and Causal Inference,\\
University of California, Berkeley
}
\date{\today}

\begin{document}
\maketitle

\begin{abstract}
We study semisupervised mean estimation with a small labeled sample, a large unlabeled sample, and a black-box prediction model whose output may be miscalibrated. A standard approach in this setting is augmented inverse-probability weighting (AIPW) \citep{robins1994estimation}, which protects against prediction-model misspecification but may be inefficient when the prediction score is poorly aligned with the outcome scale. We introduce \emph{Calibrated Prediction-Powered Inference}, which post-hoc calibrates the prediction score on the labeled sample before using it for semisupervised estimation. This simple step requires no retraining and can \emph{calibeat} the original score, improving it both as a predictor of the outcome and as a regression adjustment for semisupervised inference. We study both linear and isotonic calibration. For isotonic calibration, we establish first-order optimality guarantees: isotonic post-processing can improve predictive accuracy and estimator efficiency relative to the original score and simpler post-processing rules, while replacing the fitted isotonic score by its conditional-mean oracle yields no additional first-order variance reduction. For linear calibration, we show first-order equivalence to PPI++. We also clarify the relationship between existing estimators, showing that the original PPI estimator is a special case of AIPW and can be inefficient when the prediction model is accurate, while PPI++ is AIPW with empirical efficiency maximization \citep{rubin2008empirical}. In simulations and real-data experiments, our calibrated estimators often outperform PPI and are competitive with, or outperform, AIPW and PPI++. We provide an accompanying \texttt{Python} package, \href{https://larsvanderlaan.github.io/ppi-aipw/}{\texttt{ppi\_aipw}}.
\end{abstract}

\section{Introduction}

We study semisupervised mean estimation with one labeled dataset \(\{(X_i, Y_i)\}_{i=1}^n\) and one unlabeled dataset \(\{\widetilde X_i\}_{i=1}^N\), where outcomes are observed only for a subset sampled completely at random. This is a classical missing-outcome problem: rich covariate information is available for many individuals, but the outcome of interest is observed only for a smaller subset. Such settings arise when gold-standard outcomes are expensive to collect but prediction scores are readily available, for example in randomized trials with prognostic models trained on historical data \citep{schuler2022increasing}, biobank analyses using phenotype predictors built from EHR or omics measurements \citep{hong2019semisupervised,zhang2019phecap}, and computational biology pipelines that supplement slow or costly experimental measurements with model-based predictions \citep{ppi2023,crossppi2024}.

The goal is to use the larger unlabeled sample to improve estimation of a population mean while accounting for prediction error. This problem has long been studied in missing-data and semiparametric statistics, where methods such as AIPW, TMLE, and related debiased estimators combine flexible machine-learning prediction with valid inference \citep{robins1994estimation,vanderlaanunified,van2006targeted,vanderLaanRose2011,DoubleML}. More recently, closely related ideas have appeared in the machine-learning literature under the name prediction-powered inference \citep{ppi2023,angelopoulos2023ppi,crossppi2024}. Similar questions also arise in covariate adjustment for randomized trials, where machine-learning predictors trained on auxiliary data can improve precision in small samples \citep{hansen2008prognostic,rubin2008empirical,moore2009covariate,lin2013agnostic,benkeser2021improving,schuler2022increasing,demirel2024prediction,jin2026prognostic}.

A standard approach is to use the labeled sample to debias a plug-in estimator based on a prediction model. If \(m(X)\) is a black-box prediction score, the plug-in estimator averages \(m(X)\) over the pooled covariate sample, while the labeled sample is used to estimate the residual correction term \(\E\{Y-m(X)\}\). For example, the augmented inverse-probability weighted (AIPW) estimator \citep{robins1994estimation} averages \(m(X)\) over both labeled and unlabeled samples and then adds the average residual over the labeled sample, whereas the standard PPI estimator \citep{ppi2023} averages \(m(X)\) only over the unlabeled covariates before applying the same residual correction. Using only the unlabeled covariates can reduce efficiency when the raw score is already a good regression adjustment.

These estimators protect against bias from an imperfect prediction model, but they typically treat \(m(X)\) as fixed. This can be problematic because efficiency depends on how well \(m(X)\) serves as a regression adjustment, that is, on how well it approximates the conditional mean of \(Y\), not merely on whether it ranks observations correctly. A score may be informative for ranking individuals while still being poorly calibrated, in the sense that its numerical values systematically overstate or understate the outcome. In that case, the residuals \(Y-m(X)\) remain unnecessarily variable, and the resulting semisupervised estimator can be inefficient. Such miscalibration may arise because the model was trained on a different population, or because of temporal drift, distribution shift, or regularization bias.

This motivates post-hoc calibration. Rather than correcting only the marginal bias \(\E\{Y-m(X)\}\), one can use the labeled sample to transform \(m(X)\) so that its numerical values better align with the outcomes they are meant to predict \citep{zadrozny2001obtaining,zadrozny2002transforming,niculescu2005predicting}. An ideal target is a \textit{perfectly calibrated} score \(m^\star(X)\) satisfying
\[
\E\!\left[Y \mid m^\star(X)\right] = m^\star(X).
\]
Post-hoc calibration can improve the regression adjustment, and thus efficiency, without retraining the underlying model. Classical post-processing methods include Platt scaling \citep{platt1999probabilistic}, linear calibration \citep{mincer1969evaluation}, histogram binning \citep{zadrozny2001obtaining}, and isotonic regression \citep{zadrozny2002transforming}. Such methods are especially relevant for modern machine-learning predictors, such as neural networks, which are often informative yet systematically miscalibrated \citep{bella2010calibration,guo2017calibration,wang2023calibration}.

\paragraph{Contributions.}
We make three main contributions.

First, we introduce \emph{calibrated prediction-powered inference}
for semisupervised mean estimation (Section~\ref{sec:calppi}).
The approach uses labeled data to calibrate an existing
prediction score before using it for estimation. The resulting
plug-in estimator admits an exact AIPW representation.
We focus on linear and isotonic calibration, which are
simple to implement, require no tuning, and do not require
retraining the original predictor.

Second, we show how calibration can improve efficiency when
a prediction score is informative but poorly aligned with
the outcome scale. We study two related properties:
improving the score as a predictor of the outcome and
improving its use in semisupervised mean estimation.
Following the terminology of \citet{foster2023calibeating},
we call improvement through calibration \emph{calibeating}.
Here, these improvements concern squared prediction error
and the first-order variance of semisupervised mean estimation.

Third, we establish asymptotic guarantees for linear and
isotonic calibration. We show that linearly calibrated PPI
is first-order equivalent to unconstrained PPI++, connecting
both methods to prognostic-score regression adjustment.
For isotonic calibration, we establish asymptotic normality,
valid Wald inference, and efficiency guarantees
(Section~\ref{sec:theory}). In particular, isotonic-calibrated
PPI achieves the smallest asymptotic variance among
nondecreasing transformations of the original score. We further
show that replacing the fitted isotonic score by its
conditional-mean oracle yields no additional first-order variance reduction.

Our calibrated PPI framework can be viewed as a semisupervised
specialization of the calibrated debiased machine learning
framework developed by \citet{van2024automatic}. Our
estimators can also be interpreted as targeted minimum
loss estimators \citep{van2006targeted,vanderLaanRose2011},
with calibration used in place of a conventional targeting step.

Section~\ref{sec:setup} reviews the two-sample semisupervised
model, the AIPW, PPI, and PPI++ estimators, and the
relevant efficiency theory. Section~\ref{sec:exp} presents
simulation studies and real-data experiments.
The accompanying Python package and reproduction code
are publicly available at
\href{https://github.com/Larsvanderlaan/ppi-aipw}
{github.com/Larsvanderlaan/ppi-aipw},
with examples and API documentation at
\href{https://larsvanderlaan.github.io/ppi-aipw/}
{larsvanderlaan.github.io/ppi-aipw/}.
Appendix~\ref{appendix:code} links to the implementations and package documentation.

\subsection{Related work}

PPI is closely connected to the literature on semiparametric
estimation with missing outcomes, regression adjustment, and
debiased machine learning. In particular, using machine-learning
predictions to improve estimation efficiency is closely related
to covariate adjustment in randomized trials and causal inference
\citep{moore2009covariate,benkeser2021improving,hojbjerre2025powering}.
The power-tuning approach of PPI++ can also be understood as
empirical efficiency maximization \citep{rubin2008empirical}.
Our approach builds on these connections by using calibration
to improve the regression adjustment in semisupervised estimation.

\textbf{Semiparametric statistics and missing data.} This paper builds on foundational ideas from the missing-data and semiparametric literatures. \citet{rubin1976} formalized missing at random as a key identification condition for recovering full-data targets from partially observed outcomes. Semiparametric efficiency theory then characterized efficient influence functions, regular asymptotically linear estimators, and one-step debiased procedures \citep{levit1975efficiency,pfanzagl1985contributions,hasminskii1979nonparametric,klaassen1987consistent,bickel1993efficient}. In missing-data and causal-inference problems, these ideas led to influence-function-based estimating equations \citep{robins1994estimation,robins1995analysis,robins1995semiparametric,vanderlaanunified,tsiatis2006semiparametric}, often with doubly robust structure \citep{bang2005doubly}. As a result, the efficiency theory for missing-data models under missingness or coarsening at random is now well understood, and the class of regular asymptotically linear estimators in these models can be characterized through their influence functions \citep{robins1994estimation,robins1995analysis,robins1995semiparametric,vanderlaanunified,tsiatis2006semiparametric}.

\textbf{Debiased machine learning and targeted learning.}
A large body of work combines the semiparametric framework above with machine-learning-based nuisance estimation, often under the label of debiased machine learning. Flexible learning methods within estimating equations and AIPW-style procedures appear in early work such as \citet{vanderlaanunified}. Targeted minimum loss estimation (TMLE), or targeted learning, updates an initial nuisance estimator in a parameter-targeted way to produce a plug-in estimator with reduced first-order bias; in the present setting, the primary nuisance estimated by machine learning is the outcome regression \citep{van2006targeted,vanderLaanRose2011,gruber2009targeted,hoffman2020tmle,ross2025constructing}. These approaches are often combined with sample splitting and cross-fitting to accommodate generic machine-learning estimators \citep{schick1986asymptotically,klaassen1987consistent,zheng2010asymptotic,vanderLaanRose2011,DoubleML}. Double machine learning emphasizes similar ingredients---orthogonality, cross-fitting, and influence-function expansions---for inference with modern nuisance estimators \citep{DoubleML,chernozhukov2022automatic}, and is a closely related semiparametric framework \citep{vanderlaan2019cvtmle,diaz2020machine,chen2026equivalence}. For sufficiently stable estimators, cross-fitting can itself be relaxed \citep{chen2022stable}. Closely related work studies balancing-weight estimators \citep{hainmueller2012entropy,imai2014covariate,zubizarreta2015stable,lendle2015balancing,chattopadhyay2020balancing,hejazi2023revisiting}, which often admit equivalent regression-adjustment or augmented representations \citep{chattopadhyay2023implied,bruns2025augmented,rotnitzky2025note}. In missing-at-random settings with known missingness probability, including semisupervised learning and randomized trials, doubly robust estimators typically permit valid inference under flexible nuisance estimation even when the learned regression adjustment is misspecified \citep{moore2009covariate,moore2009increasing,rosenblum2009using,hojbjerre2025powering,hojbjerre2026within}.

\textbf{Semisupervised inference and improving power in randomized trials.}
Related work studies semisupervised or surrogate-assisted estimation in settings where the primary outcome is observed only on a limited subset, while auxiliary outcomes or surrogate measurements are available more broadly \citep{pepe1992inference,pepe1994auxiliary,chen2000miscellanea,chen2008improving,cheng2021robust,ji2025predictions,kallus2025role}. Two-sample semisupervised inference is a special case of two-stage sampling designs \citep{scott1982effect,rose2011targeted,hejazi2021efficient,qiu2026efficient} and, more generally, of data fusion, for which a substantial semiparametric efficiency literature has been developed \citep{li2023efficient,li2025data,graham2024towards,xu2025unified}. More recent papers have used the term \emph{prediction-powered inference} (PPI) for the semisupervised mean-estimation setting in which a black-box predictor is combined with a small labeled sample and a large unlabeled sample \citep{ppi2023,angelopoulos2023ppi,crossppi2024,xu2025unified,song2026demystifying,poulet2025prediction,wang2025efficient,lee2026mec}. These papers also brought such methods to a broader machine-learning audience and provided accessible software. Closely related questions have long been studied in randomized trials, where prognostic scores or other regression adjustments learned from larger historical, external-control, or auxiliary data sources are used to improve precision in smaller trials \citep{hansen2008prognostic,rosenblum2009using,moore2009covariate,moore2009increasing,moore2011robust,vanderweele2019selecting,benkeser2021improving,burger2021use,schuler2022increasing,li2023estimating,balzer2024adaptive,van2024covariate,hojbjerre2025powering,yang2026improving}. Methodologically, however, PPI estimators can often be viewed as instances or refinements of existing semiparametric procedures. For example, PPI++ is an AIPW estimator \citep{robins1994estimation,vanderlaanunified} obtained by empirical efficiency maximization \citep{rubin2008empirical} over a scaling parameter, while cross-PPI \citep{crossppi2024} is a cross-fitted AIPW estimator \citep{zheng2010asymptotic,DoubleML}. More fundamentally, in settings with known missingness probability, the ability to combine black-box prediction models with valid semiparametric inference is closely tied to the robustness properties of AIPW and is therefore not unique to semisupervised mean estimation \citep{bang2005doubly,seaman2018introduction,rotnitzky2021characterization}.

\textbf{Calibrated DML and Recalibrated PPI.}
In our mean-estimation setting, the closest precursor is the calibrated debiased machine-learning framework of \citet{van2024automatic}, which studies calibration of the outcome regression in AIPW estimators for doubly robust inference. Our setting can be viewed as a semisupervised specialization of that perspective, adapted to prediction-powered inference and aimed at improving efficiency in randomized designs. This connection also motivates our term \emph{calibrated PPI}. A related method is recalibrated prediction-powered inference (RePPI) \citep{ji2025predictions}, which applies to general estimating equations. Rather than correcting the predictive model itself, their approach modifies the imputed loss in the PPI estimating equation, thereby ``recalibrating” the estimator rather than the predictions. In the mean-estimation setting under squared loss, the target correction reduces to the conditional mean of the outcome given the initial prediction score, which is also the efficient regression adjustment in the reduced model based only on that score. Thus, this conditional mean may be viewed as an oracle calibrator. The main differences lie in the estimators, emphasis, and theory. \citet{ji2025predictions} use flexible machine learning to estimate this oracle target and thereby pursue semiparametric efficiency \citep{robins1994estimation}. They compare linear, random-forest, and neural-network recalibration, while retaining an AIPW estimator. Our focus is on simple post-hoc calibration maps applied directly to a fixed score and on guarantees that do not require consistent estimation of the oracle target. Our estimators retain a simple plug-in form, yielding lightweight methods that can be incorporated into an existing prediction pipeline without retraining. We also develop calibration-specific results: an exact AIPW interpretation of the plug-in estimator, first-order equivalence between linear calibration and PPI++, and a first-order calibeating guarantee for isotonic calibration.

\textbf{From missing data to causal inference.} From the perspective of estimation and efficiency theory, semisupervised mean estimation with missing-at-random outcomes is equivalent to randomized experiments, controlled trials, and survey-sampling designs once the missingness indicator is reinterpreted as a treatment or sampling indicator, so that the relevant potential outcomes are viewed as missing \citep{rubin1978bayesian,vanderlaanunified,licausal,mozer2026ppi}. This equivalence allows tools from covariate adjustment, causal inference, and model-assisted estimation to be transferred directly to the semisupervised setting; see, for example, \citet{mozer2026ppi} for results relating prediction-powered inference to model-assisted survey-sampling estimators. The two-sample formulation we study is also equivalent to a one-sample i.i.d.\ model \citep{li2023efficient}. In particular, our methods apply to inference on the counterfactual mean \(\mathbb{E}[Y(1)]\) in a randomized trial with binary treatment, where treated units form the labeled sample and control units contribute covariate information only; see Appendix~\ref{app:causal}.

\textbf{Calibration.}  Classical post-processing methods such as Platt scaling \citep{platt1999probabilistic,cox1958two,gupta2023online}, linear calibration \citep{mincer1969evaluation,chernozhukov2018generic,leng2021calibration}, isotonic regression \citep{zadrozny2002transforming,niculescu2005predicting}, and histogram binning \citep{zadrozny2001obtaining,gupta2021distribution} improve predictive reliability without retraining the underlying model. The idea has deeper roots in the forecasting literature \citep{mincer1969evaluation,lichtenstein1977calibration,vovk1992universal,vovk2005algorithmic,gneiting2007probabilistic,lambert2011elicitation}; see also the historical discussion in \citet{lee2023t}. More recent work studies calibration under broader losses and prediction tasks \citep{jung2021moment,noarov2023statistical,van2023causal,whitehouse2024orthogonal,van2025venn}, as well as applications to conformal prediction \citep{van2024self,van2025venn}, importance-weight stabilization \citep{gutman2022propensity,deshpande2023calibrated,ballinari2024improving,van2024stabilized}, treatment effect estimation \citep{chernozhukov2018generic,van2023causal,whitehouse2024orthogonal}, and value prediction in offline reinforcement learning \citep{van2025bellman,van2025semiparametric}. Distribution-free guarantees for histogram binning are developed by \citet{gupta2021distribution}, while \citet{van2023causal,van2025venn} establish asymptotic distribution-free guarantees for isotonic calibration.

 \section{Two-Sample Semisupervised Setup}
\label{sec:setup}

\subsection{Data structure and notation}

We observe two independent datasets,
\[
\mathcal D_L=\{(X_i,Y_i)\}_{i=1}^n
\qquad\text{and}\qquad
\mathcal D_U=\{\widetilde X_j\}_{j=1}^N.
\]
The labeled sample consists of i.i.d.\ draws from a distribution \(P_0\) on \(\mathcal X\times\mathbb R\), while the unlabeled sample consists of i.i.d.\ draws from the distribution \(\Px\) of \(X\) under \(P_0\). Our target is the population mean
\[
\psi_0 := \E_{P_0}[Y],
\qquad
\muzero(x):=\E_{P_0}[Y\mid X=x].
\]
This can also be written as
\[
\psi_0=\int y\,dP_0(x,y)=\int \muzero(x)\,d\Px(x).
\]

For functions \(a:\mathcal X\times\mathbb R\to\mathbb R\) and \(b:\mathcal X\to\mathbb R\), define the empirical means over the labeled and unlabeled samples by
\[
\Pl a := \frac{1}{n}\sum_{i=1}^n a(X_i,Y_i),
\qquad
\Pu b := \frac{1}{N}\sum_{j=1}^N b(\widetilde X_j).
\]
For readability, we will sometimes write \(\Pl\{a(X,Y)\}\) instead of \(\Pl a\). When a function depends only on \(X\), we also write
\[
\Pl\{b(X)\}:=\frac{1}{n}\sum_{i=1}^n b(X_i).
\]
We write expectations and variances as \(\E_{P_0}[\cdot]\), \(\E_{\Px}[\cdot]\), \(\Var_{P_0}[\cdot]\), and \(\Var_{\Px}[\cdot]\).

Throughout, let \(M:=n+N\) and \(\rho_n:=n/M\) denote the labeled fraction. Then
\[
\frac{1}{n+N}\left\{\sum_{i=1}^n f(X_i)+\sum_{j=1}^N f(\widetilde X_j)\right\}
=
\rho_n \Pl\{f(X)\} + (1-\rho_n)\Pu\{f(\widetilde X)\}.
\]
In other words, the empirical mean of \(f(X)\) over the pooled covariate sample is a weighted average of the empirical means from the labeled and unlabeled samples.

\begin{assumption}[Two-sample design]
\label{assump:design}
The following hold:
\begin{enumerate}[leftmargin=6mm,label=(\roman*)]
\item \(\mathcal D_L\) is i.i.d.\ from \(P_0\), and \(\mathcal D_U\) is i.i.d.\ from \(\Px\).
\item \(\mathcal D_L\) and \(\mathcal D_U\) are independent.
\item \(Y\) has finite second moment under \(P_0\).
\item As \(n\to\infty\), the labeled fraction \(\rhon=n/(n+N)\) satisfies \(\rhon\to\rhozero\) for some \(\rhozero\in[0,1)\).
\end{enumerate}
\end{assumption}

\subsection{Review of AIPW, PPI, and semiparametric efficiency}
\label{sec:effreview}

\textbf{The AIPW class.}
We begin with augmented inverse-probability weighting (AIPW) in the
unrestricted two-sample model \citep{robins1994estimation}. Rather
than a single estimator, AIPW defines a class indexed by a
square-integrable prediction function $f:\mathcal X\to\R$.
For any such $f$, an AIPW estimator is given by
\begin{equation}
\label{eq:aipw-family}
\widehat\psi(f)
:=
\rhon \Pl\{f(X)\}
+
(1-\rhon)\Pu\{f(\widetilde X)\}
+
\Pl\{Y-f(X)\}.
\end{equation}
The estimator uses both samples to estimate the mean of $f(X)$
and the labeled sample to estimate the mean residual.
For any fixed $f$, $\widehat\psi(f)$ is unbiased in finite samples
for $\psi_0=\E_{P_0}[Y]$. In particular, setting $f=0$ recovers
the labeled-sample mean $\Pl\{Y\}$. The choice of $f$ therefore
affects precision but not unbiasedness. A poorly specified
prediction function may forgo substantial efficiency gains, but
does not introduce bias. In this sense, AIPW is a ``safe''
estimator \citep{xu2025unified}. This robustness to the choice
of $f$ is closely related to the robustness to regression
misspecification in doubly robust estimation
\citep{bang2005doubly}.\footnote{With unknown, covariate-dependent
missingness probabilities, AIPW generally no longer enjoys
finite-sample unbiasedness, but retains the usual asymptotic
double robustness property \citep{seaman2018introduction}.}

The central question is therefore how to choose $f$ to minimize
variance. When $\rhon$ is the labeled fraction of the combined
sample, the variance-minimizing choice is
$f_0(x)=\E_{P_0}[Y\mid X=x]$, up to an additive constant.
Efficiency thus depends on the accuracy of conditional mean
predictions. This principle also motivates the use of flexible
regression methods for covariate adjustment in randomized trials
\citep{moore2009covariate,moore2009increasing,schuler2022increasing,hojbjerre2025powering}.

\textbf{Asymptotic linearity, efficiency, and the role of residual variance.} To characterize the leading large-sample error in the two-sample model, we decompose the estimation error into one contribution from the labeled sample and one from the unlabeled sample.We say that
$\widehat\psi$ is asymptotically linear with triangular influence
pair $(D_n^L,D_n^U)$ if
\[
\widehat\psi-\psi_0
=
\frac{1}{n}\sum_{i=1}^n D_n^L(X_i,Y_i)
+
\frac{1}{N}\sum_{j=1}^N D_n^U(\widetilde X_j)
+
o_p(n^{-1/2}),
\]
where $\E_{P_0}[D_n^L(X,Y)]=0$ and
$\E_{\Px}[D_n^U(\widetilde X)]=0$.
If the influence functions converge in $L^2$ and $n/N$ converges
to a finite limit, the central limit theorem and independence
of the two samples yield
\[
\sqrt n\,(\widehat\psi-\psi_0)
\rightsquigarrow
N\!\left(0,\,
\lim_{n\to\infty}\left[
\Var_{P_0}\{D_n^L(X,Y)\}
+
\frac{n}{N}\Var_{\Px}\{D_n^U(\widetilde X)\}
\right]\right).
\]
The influence pair thus determines the asymptotic variance,
accounting for variability in both samples. The following
result characterizes the influence pair for the AIPW class,
identifies the efficient choice of $f$, and derives the
semiparametric efficiency bound in the unrestricted two-sample
model \citep{bickel1993efficient,robins1994estimation,robins1995semiparametric}.

\begin{proposition}[Influence pair and efficiency of AIPW]
\label{prop:aipw-class}
Under \cref{assump:design}, for any fixed square-integrable
function $f:\mathcal X\to\R$, the estimator $\widehat\psi(f)$
is asymptotically linear with triangular influence pair
$(D_{f,n}^L,D_{f,n}^U)$, where
\begin{align*}
D_{f,n}^L(X,Y)
&=
Y-\psi_0-(1-\rho_n)\{\tilde f(X)-\psi_0\},\\
D_{f,n}^U(\widetilde X)
&=
(1-\rho_n)\{\tilde f(\widetilde X)-\psi_0\},
\end{align*}
and $\tilde f(X):=f(X)-\E_{\Px}[f(X)]+\psi_0$
is the recentered prediction function. The asymptotic variance
of $\sqrt n\{\widehat\psi(f)-\psi_0\}$ is
\[
\rhozero\Var_{P_0}[\mu_0(X)]
+
\E_{P_0}\!\bigl[\Var[Y\mid X]\bigr]
+
(1-\rhozero)\Var_{P_0}[f(X)-\mu_0(X)].
\]
This variance is minimized when $f=\mu_0$ up to an additive
constant. Consequently, the efficient triangular influence
pair in the unrestricted two-sample model is
$(D_{\mu_0,n}^L,D_{\mu_0,n}^U)$, with efficiency bound $\rhozero\Var_{P_0}[\mu_0(X)]
+
\E_{P_0}\!\bigl[\Var[Y\mid X]\bigr].$
\end{proposition}

Proposition~\ref{prop:aipw-class} shows that the excess
asymptotic variance of AIPW relative to the efficiency bound is
\[
(1-\rhozero)\Var_{P_0}[f(X)-\mu_0(X)].
\]
Efficiency therefore depends on the accuracy of $f$ as an
approximation to the conditional mean $\mu_0$, up to an additive
constant \citep{wang2025efficient}. A score may rank observations
correctly yet remain inefficient if its numerical values are
poorly aligned with the conditional mean. This motivates the use of flexible black-box prediction scores
for regression adjustment, together with post-hoc calibration
to improve their approximation to the conditional mean.

Two familiar estimators arise from different choices of $f$ in
\eqref{eq:aipw-family}. Taking $f=m$, where $m(X)$ is a prediction
score, yields the usual AIPW estimator \citep{robins1994estimation},
\[
\widehat{\psi}_{\mathrm{AIPW}}
=
\frac{1}{M}\left\{
\sum_{i=1}^n m(X_i)
+
\sum_{j=1}^N m(\widetilde X_j)
+
\frac{1}{\rho_n}\sum_{i=1}^n \{Y_i-m(X_i)\}
\right\},
\]
where $\rho_n^{-1}$ is the inverse probability of observing the
outcome.

Prediction-powered inference (PPI), introduced by \citet{ppi2023},
instead uses\footnote{The PPI and PPI++ estimators are algebraically
equivalent to the model-assisted survey-design estimators of
\citet{cassel1976some} and \citet{sarndal2003model}, respectively;
see \citet{mozer2026ppi} for a review.}
\[
\widehat{\psi}_{\mathrm{PPI}}
=
\Pu\{m(\widetilde X)\}+\Pl\{Y-m(X)\},
\]
which corresponds to $f=m/(1-\rho_n)$ in the AIPW class.
Unlike AIPW, PPI averages predictions only over the unlabeled
sample, effectively rescaling the prediction score by
$(1-\rho_n)^{-1}>1$. By Proposition~\ref{prop:aipw-class},
this rescaling improves efficiency only if it reduces the
variance of the prediction error relative to $\mu_0$.
In particular, when $\rhozero>0$ and $\mu_0(X)$ is nonconstant,
PPI is inefficient even when $m=\mu_0$, whereas AIPW attains
the efficiency bound. AIPW is therefore the natural default
when $m$ is intended to approximate the conditional mean
\citep{han2012note,rotnitzky2012improved}.

\textbf{Empirical efficiency maximization.}
Both AIPW and PPI belong to the scaling class
$\{\lambda m(X):\lambda\in\mathbb R\}$, corresponding to
$\lambda=1$ and $\lambda=(1-\rho_n)^{-1}$, respectively.
Rather than fixing the scaling factor, one can choose it by
minimizing an empirical estimate of the asymptotic variance,
a procedure known as empirical efficiency maximization
\citep{rubin2008empirical,vanderlaanunified}.
The empirical variance criterion is
\[
\widehat\sigma_n^2(f)
:=
\widehat{\Var}_L\{Y-(1-\rho_n)f(X)\}
+\rho_n(1-\rho_n)\widehat{\Var}_U\{f(\widetilde X)\},
\]
where \(\widehat{\Var}_L(a):=\Pl[(a-\Pl a)^2]\) and
\(\widehat{\Var}_U(b):=\Pu[(b-\Pu b)^2]\).
Minimizing $\widehat\sigma_n^2(\lambda m)$ over
$\lambda\in\mathbb R$ yields an AIPW estimator with the
rescaled prediction function $\hat\lambda m$. This is
equivalent, up to parameterization, to the PPI++ estimator
of \citet{angelopoulos2023ppi}. The population-optimal
scaling coefficient is
\begin{equation}
\label{opt:scale}
\lambda^\star
=
\frac{\Cov_{P_0}\{Y,m(X)\}}{\Var_{P_0}\{m(X)\}},
\end{equation}
provided $\Var_{P_0}\{m(X)\}>0$. Empirical efficiency maximization also extends to richer
classes of prediction functions. For example, one can
minimize $\widehat\sigma_n^2(f)$ over
\[
\left\{
\lambda m(X)+\beta_{\mathrm{lin}}^\top X:
\lambda\in\mathbb R,\
\beta_{\mathrm{lin}}\in\mathbb R^d
\right\},
\]
over monotone transformations of $m$, or over suitably
regularized reproducing kernel Hilbert spaces; see
Appendix~\ref{appendix:remarks}. Given multiple prediction
scores $m_1,m_2,\ldots$, the same criterion can be used
to learn an efficient linear combination. Cross-validation
can further be used to select among candidate classes
\citep{rubin2008empirical}.
\section{Calibrated Prediction-Powered Inference}
\label{sec:calppi}

Standard AIPW and PPI estimators use a prediction score $m(X)$
directly for regression adjustment. Although $m(X)$ may rank
observations well, its numerical values need not accurately
approximate the conditional mean of the outcome. As
Proposition~\ref{prop:aipw-class} shows, poor calibration
can lead to substantial efficiency losses, even for highly
predictive scores.

We propose calibrating $m(X)$ using the labeled sample before
using it for semisupervised estimation. We call this approach
\emph{calibrated prediction-powered inference}. The resulting
estimator averages calibrated predictions over the pooled
labeled and unlabeled samples and admits an exact AIPW
representation. We study linear and isotonic calibration,
which use affine and nondecreasing transformations of the
original score, respectively.

\subsection{Calibrated PPI and its AIPW representation}

Let $m:\mathcal X\to\mathbb R$ be a black-box regression model
fit using external data or a separate training sample.
Given a class $\mathcal H$ of functions $h:\mathbb R\to\mathbb R$,
we estimate a calibration map by empirical risk minimization
on the labeled sample:
\begin{equation}
\label{eq:H-erm}
f_n \in \arg\min_{h\in\mathcal H}
\Pl\bigl[(Y-h(m(X)))^2\bigr].
\end{equation}
Define the calibrated score $\mtilde:=f_n\circ m$.
Our calibrated PPI estimator averages these predictions
over the pooled labeled and unlabeled samples:
\begin{equation}
\label{eq:plugin}
\psihat
:=
\frac{1}{n+N}\left\{
\sum_{i=1}^n \mtilde(X_i)
+
\sum_{j=1}^N \mtilde(\widetilde X_j)
\right\}.
\end{equation}
Equivalently,
$\psihat
=
\rhon\Pl\{\mtilde(X)\}
+
(1-\rhon)\Pu\{\mtilde(\widetilde X)\}$.

If $\mathcal H$ is closed under the addition of constants,
the first-order condition for \eqref{eq:H-erm} gives
\[
\Pl\{Y-\mtilde(X)\}=0.
\]
Consequently, the plug-in estimator \eqref{eq:plugin}
admits the exact AIPW representation
\begin{equation}
\label{eqn::aipwcal}
\psihat
=
\rhon \Pl\{\mtilde(X)\}
+
(1-\rhon)\Pu\{\mtilde(\widetilde X)\}
+
\Pl\{Y-\mtilde(X)\}.
\end{equation}
The exact AIPW representation does not imply exact unbiasedness when
\(\mtilde\) is fitted using the same labeled outcomes. Because calibration
uses only the labeled sample, conditioning on that sample gives, whenever
the expectations are finite,
\[
\E(\psihat-\psi_0)
=-(1-\rho_n)\E\bigl[(\Pl-\Px)\mtilde\bigr].
\]
The right-hand side is zero for a fixed score, but need not vanish for a
fitted score. The results below establish asymptotic linearity for linear
and isotonic calibration without sample splitting.
Section~\ref{sec:crossfit} gives exactly unbiased cross-fitted and
leave-one-out AIPW estimators that retain the residual correction and
exclude each observation from its calibration fit.

The goal is to obtain a calibrated score satisfying, at least
approximately,
\[
\mtilde(X)=\E_{P_0}[Y\mid\mtilde(X)].
\]
At the population level, this property ensures that no further
transformation of the calibrated score can reduce mean squared
prediction error. By Proposition~\ref{prop:aipw-class}, it also
ensures that no further post-processing can improve the
asymptotic efficiency of the corresponding AIPW estimator. At the same time, we want to ensure that the calibrated prediction $\mtilde(X)$ is no worse, at least asymptotically, than the original score $m(X)$ when used in the AIPW estimator.

For a fixed score \(q\), the asymptotic variance of
\(\sqrt n\{\widehat\psi(q)-\psi_0\}\) can also be written as
\[
V(q)=\rho_0\Var(Y)+(1-\rho_0)\Var\{Y-q(X)\}.
\]
Adding a constant to \(q\) leaves \(V(q)\) unchanged.
Thus, if $\mathcal H$ is closed under addition of constants,
the population least-squares minimizer also minimizes the
asymptotic variance of AIPW over the same class of transformed scores.
If the identity map belongs to $\mathcal H$, the resulting
asymptotic variance is no greater than that of AIPW based on $m(X)$.
If $\mathcal H$ contains the conditional mean of $Y$ given $m(X)$,
the population minimizer attains the efficiency bound in the reduced
model based on $(m(X),Y)$. The choice of
$\mathcal H$ balances flexibility to correct miscalibration
against the variability of estimating the transformation
from limited labeled data. We consider affine and
nondecreasing transformations, corresponding to linear
and isotonic calibration, respectively.

\subsection{Isotonic-calibrated PPI}
\label{sec:iso}

We consider isotonic regression as a nonparametric post-hoc
calibration method
\citep{zadrozny2001obtaining,zadrozny2002transforming,niculescu2005predicting}.
Let $\mathcal F_{\mathrm{iso}}$ denote the class of nondecreasing
real-valued functions. We fit the calibration map using the
labeled pairs $\{(\mhat(X_i),Y_i)\}_{i=1}^n$ by solving
\[
\hat f \in \arg\min_{f \in \mathcal F_{\mathrm{iso}}}
\sum_{i=1}^n \bigl\{Y_i - f(\mhat(X_i))\bigr\}^2,
\]
and define the calibrated score
\[
m_{n,\mathrm{iso}}^\star(X) = \hat f\{\mhat(X)\}.
\]
The resulting isotonic-calibrated PPI estimator is
\[
\widehat{\psi}_{\mathrm{iso}}
:=
\frac{1}{n+N}
\left\{
\sum_{i=1}^n m_{n,\mathrm{iso}}^\star(X_i)
+
\sum_{j=1}^N m_{n,\mathrm{iso}}^\star(\widetilde X_j)
\right\}.
\]

Isotonic regression is tuning-free and can be computed
efficiently using the pooled-adjacent-violators algorithm
(PAVA) \citep{barlow1972isotonic}. The fitted values on the labeled sample are constant on PAVA blocks,
with each block value equal to the empirical mean outcome in that block.
For new score values, we interpolate linearly between the fitted values
at distinct observed scores and use the endpoint values outside the observed range. Consequently, no further transformation of
the fitted score can reduce empirical squared prediction
error. Thus, isotonic regression achieves full empirical
calibration despite restricting the fitted transformation
to be monotone. We establish its asymptotic calibration
and efficiency properties in Section~\ref{sec:theory}.

\begin{proposition}[Isotonic calibration is fully calibrated]
\label{prop:iso-repr}
For every function \(h:\R\to\R\),
\begin{equation}
\label{eq:orth}
\Pl\bigl[h(m_{n,\mathrm{iso}}^\star(X))\{Y-m_{n,\mathrm{iso}}^\star(X)\}\bigr]=0.
\end{equation}
Moreover, for any \(\theta \in \mathcal{F}_{\mathrm{iso}}\),
\[
\Pl\bigl[(Y-m_{n,\mathrm{iso}}^\star(X))^2\bigr]
\le
\Pl\bigl[(Y-\theta(\mhat(X)))^2\bigr].
\]
In particular,
\[
\Pl\bigl[(Y-m_{n,\mathrm{iso}}^\star(X))^2\bigr]
\le
\Pl\bigl[(Y-\mhat(X))^2\bigr].
\]
\end{proposition}

\begin{algorithm}[t]
\caption{Isotonic-calibrated PPI}
\label{alg:isoppi}
\begin{algorithmic}[1]
\Require \(\mathcal D_L=\{(X_i,Y_i)\}_{i=1}^n\), \(\mathcal D_U=\{\widetilde X_j\}_{j=1}^N\), score \(\mhat\), level \(1-\alpha\)
\State Compute \(T_i=\mhat(X_i)\) and \(\widetilde T_j=\mhat(\widetilde X_j)\)
\State Fit \(\hat f\in\arg\min_{f\in\mathcal F_{\mathrm{iso}}}\sum_{i=1}^n\{Y_i-f(T_i)\}^2\)
\State Set \(m_{n,\mathrm{iso}}^\star(X_i)=\hat f(T_i)\) and \(m_{n,\mathrm{iso}}^\star(\widetilde X_j)=\hat f(\widetilde T_j)\)
\State Form \(\widehat\psi_{\mathrm{iso}}\gets (n+N)^{-1}\{\sum_{i=1}^n m_{n,\mathrm{iso}}^\star(X_i)+\sum_{j=1}^N m_{n,\mathrm{iso}}^\star(\widetilde X_j)\}\)
\State Set \(\widehat D_i^L=Y_i-\widehat\psi_{\mathrm{iso}}-(1-\rhon)\{m_{n,\mathrm{iso}}^\star(X_i)-\widehat\psi_{\mathrm{iso}}\}\) and \(\widehat D_j^U=(1-\rhon)\{m_{n,\mathrm{iso}}^\star(\widetilde X_j)-\widehat\psi_{\mathrm{iso}}\}\)
\State Set \(\widehat{\mathrm{SE}}_{\mathrm{iso}}^2=n^{-1}\widehat{\Var}_L(\widehat D^L)+N^{-1}\widehat{\Var}_U(\widehat D^U)\)
\State Return \(\widehat\psi_{\mathrm{iso}} \pm z_{1-\alpha/2}\widehat{\mathrm{SE}}_{\mathrm{iso}}\)
\end{algorithmic}
\end{algorithm}

\noindent\textbf{Implementation.}
Isotonic regression is typically fitted using the
pooled-adjacent-violators algorithm (PAVA)
\citep{barlow1972isotonic}, which is computationally efficient
but may become burdensome for very large datasets.
Alternatively, monotonicity-constrained tree-based methods,
such as XGBoost and LightGBM, can be used to estimate
monotone calibration maps; see Appendix~\ref{appendix:code}.
These methods also permit regularization through parameters
such as the number of leaves, maximum depth, and minimum
leaf size.

\noindent\textbf{Histogram binning.}
Histogram binning partitions the score range into prespecified
bins and replaces the score within each bin by the empirical
mean of the outcome
\citep{zadrozny2001obtaining,zadrozny2002transforming}.
Here, $\mathcal H$ consists of piecewise-constant functions
on a fixed partition of the score space
\citep{gupta2020distribution,gupta2021distribution}, and
one may take $\mathcal F=\mathcal H$.
Like isotonic regression, histogram binning satisfies full
empirical calibration. The main difference is that histogram
binning uses a fixed partition, whereas isotonic regression
learns the partition from the data.

\subsection{Linearly calibrated PPI}
\label{sec:linear}

When the labeled sample is small, a parametric calibration rule
can be more stable than histogram binning or isotonic regression.
We therefore consider \emph{linear calibration}, which rescales
and shifts the prediction score using an affine map. This is
similar in spirit to Platt scaling and temperature scaling
\citep{platt1999probabilistic,guo2017calibration}. Unlike
isotonic calibration, linear calibration cannot correct general
nonlinear miscalibration, but offers a simple, low-dimensional
alternative when labeled data are limited.

We estimate the calibration map by least squares on the
labeled sample:
\begin{equation}
\label{eq:linear-cal}
\bigl(\hat a_{\mathrm{lin}},\hat b_{\mathrm{lin}}\bigr)
\in\arg\min_{a,b\in\R}
\sum_{i=1}^n \{Y_i-a\,\mhat(X_i)-b\}^2,
\qquad
m_{n,\mathrm{lin}}^\star(X)
:=\hat a_{\mathrm{lin}}\,\mhat(X)+\hat b_{\mathrm{lin}}.
\end{equation}
The resulting estimator averages the calibrated predictions
over the labeled and unlabeled samples:
\begin{equation}
\label{eq:linear-plugin}
\widehat{\psi}_{\mathrm{lin}}
:= \frac{1}{n+N}
\left\{
\sum_{i=1}^n m_{n,\mathrm{lin}}^\star(X_i)
+
\sum_{j=1}^N m_{n,\mathrm{lin}}^\star(\widetilde X_j)
\right\}.
\end{equation}
By construction, no further affine transformation of the
fitted score can reduce empirical squared prediction error:
\[
\Pl\bigl[(Y-m_{n,\mathrm{lin}}^\star(X))^2\bigr]
\le
\Pl\bigl[(Y-a-b\,m_{n,\mathrm{lin}}^\star(X))^2\bigr],
\qquad a,b\in\R.
\]
Thus, $m_{n,\mathrm{lin}}^\star$ is empirically calibrated
over the class of affine transformations. Moreover, the
least-squares normal equations imply
$\Pl\{Y-m_{n,\mathrm{lin}}^\star(X)\}=0$, so
$\widehat{\psi}_{\mathrm{lin}}$ admits the exact AIPW
representation \eqref{eqn::aipwcal}. The calibration regression can also include additional
covariates when sufficient labeled data are available;
see Algorithm~\ref{alg:linearppi}.

The next result connects linear calibration to PPI++, which
rescales the original score by empirical efficiency maximization.
For mean estimation, the two methods are first-order equivalent,
although they need not coincide in finite samples. Both target
the population-optimal scaling coefficient $\lambda^\star$ in
\eqref{opt:scale}, corresponding to the slope in the population
linear regression of $Y$ on $m(X)$. This equivalence requires
optimization over all $\lambda\in\mathbb R$. Restricting
$\lambda$ to $[0,(1-\rho_n)^{-1}]$, as in the official
\texttt{ppi\_py} implementation, can break this equivalence
and reduce efficiency, as illustrated in our experiments.

Let $\widehat{\psi}_{++}(\widehat{\lambda}_{++})$ denote the
unconstrained PPI++ estimator of \citet{angelopoulos2023ppi},
namely, the AIPW estimator based on the rescaled score
$\widehat{\lambda}_{++}m(X)$, where
$\widehat{\lambda}_{++}$ is chosen by empirical efficiency
maximization over $\{\lambda m(X):\lambda\in\mathbb R\}$
\citep{rubin2008empirical}.

\begin{proposition}[First-order equivalence of PPI++ and linear calibration]
\label{prop:ppi++-linear}
Under \cref{assump:design}, suppose that
$\Var[m(X)]>0$ and $\E_{P_0}[m(X)^2]<\infty$.
Then
\[
\widehat\psi_{++}(\hat\lambda_{++})-\widehat\psi_{\mathrm{lin}}
=
(1-\rho_n)\bigl(\hat\lambda_{++}-\hat a_{\mathrm{lin}}\bigr)
\left\{\Pu\{m(\widetilde X)\}-\Pl\{m(X)\}\right\}.
\]
Moreover, $\hat\lambda_{++}-\hat a_{\mathrm{lin}}
=o_p(1)$, and
\[
\widehat\psi_{++}(\hat\lambda_{++})-\widehat\psi_{\mathrm{lin}}
=
o_p(n^{-1/2}).
\]
\end{proposition}

\begin{algorithm}[t]
\caption{Linearly calibrated PPI}
\label{alg:linearppi-basic}
\begin{algorithmic}[1]
\Require $\mathcal D_L=\{(X_i,Y_i)\}_{i=1}^n$,
$\mathcal D_U=\{\widetilde X_j\}_{j=1}^N$,
score $\mhat$, confidence level $1-\alpha$
\State Compute $T_i=\mhat(X_i)$ and
$\widetilde T_j=\mhat(\widetilde X_j)$
\State Fit
$(\hat a,\hat b)\in\arg\min_{a,b}
\sum_{i=1}^n\{Y_i-(aT_i+b)\}^2$
\State Set
$m_{n,\mathrm{lin}}^\star(X_i)=\hat aT_i+\hat b$
and
$m_{n,\mathrm{lin}}^\star(\widetilde X_j)
=\hat a\widetilde T_j+\hat b$
\State Compute
$\widehat\psi_{\mathrm{lin}}\gets
(n+N)^{-1}
\{\sum_{i=1}^n m_{n,\mathrm{lin}}^\star(X_i)
+\sum_{j=1}^N m_{n,\mathrm{lin}}^\star(\widetilde X_j)\}$
\State Set
$\widehat D_i^L=
Y_i-\widehat\psi_{\mathrm{lin}}
-(1-\rhon)\{m_{n,\mathrm{lin}}^\star(X_i)
-\widehat\psi_{\mathrm{lin}}\}$
and
$\widehat D_j^U=
(1-\rhon)\{m_{n,\mathrm{lin}}^\star(\widetilde X_j)
-\widehat\psi_{\mathrm{lin}}\}$
\State Compute
$\widehat{\mathrm{SE}}_{\mathrm{lin}}^2
=n^{-1}\widehat{\Var}_L(\widehat D^L)
+N^{-1}\widehat{\Var}_U(\widehat D^U)$
\State \Return
$\widehat\psi_{\mathrm{lin}}
\pm z_{1-\alpha/2}\widehat{\mathrm{SE}}_{\mathrm{lin}}$
\end{algorithmic}
\end{algorithm}

\begin{algorithm}[t]
\caption{Linearly calibrated PPI with covariate adjustment}
\label{alg:linearppi}
\begin{algorithmic}[1]
\Require $\mathcal D_L=\{(X_i,Y_i)\}_{i=1}^n$,
$\mathcal D_U=\{\widetilde X_j\}_{j=1}^N$,
score $\mhat$, confidence level $1-\alpha$
\State Compute $T_i=\mhat(X_i)$ and
$\widetilde T_j=\mhat(\widetilde X_j)$
\State Fit
$(\hat a,\hat\beta,\hat b)\in\arg\min_{a,\beta,b}
\sum_{i=1}^n\{Y_i-(aT_i+X_i^\top\beta+b)\}^2$
\State Set
$m_{n,\mathrm{lin}}^\star(X_i)
=\hat aT_i+X_i^\top\hat\beta+\hat b$
and
$m_{n,\mathrm{lin}}^\star(\widetilde X_j)
=\hat a\widetilde T_j+\widetilde X_j^\top\hat\beta+\hat b$
\State Compute
$\widehat\psi_{\mathrm{lin}}\gets
(n+N)^{-1}
\{\sum_{i=1}^n m_{n,\mathrm{lin}}^\star(X_i)
+\sum_{j=1}^N m_{n,\mathrm{lin}}^\star(\widetilde X_j)\}$
\State Set
$\widehat D_i^L=
Y_i-\widehat\psi_{\mathrm{lin}}
-(1-\rhon)\{m_{n,\mathrm{lin}}^\star(X_i)
-\widehat\psi_{\mathrm{lin}}\}$
and
$\widehat D_j^U=
(1-\rhon)\{m_{n,\mathrm{lin}}^\star(\widetilde X_j)
-\widehat\psi_{\mathrm{lin}}\}$
\State Compute
$\widehat{\mathrm{SE}}_{\mathrm{lin}}^2
=n^{-1}\widehat{\Var}_L(\widehat D^L)
+N^{-1}\widehat{\Var}_U(\widehat D^U)$
\State \Return
$\widehat\psi_{\mathrm{lin}}
\pm z_{1-\alpha/2}\widehat{\mathrm{SE}}_{\mathrm{lin}}$
\end{algorithmic}
\end{algorithm}

\begin{remark}[Relation to prognostic-score adjustment]
Linear calibration can be viewed as prognostic-score regression
adjustment, a classical approach to improving precision through
covariate adjustment
\citep{hansen2008prognostic,rosenblum2009using,lin2013agnostic,
schuler2022increasing,balzer2024adaptive,hojbjerre2025powering}.
Proposition~\ref{prop:ppi++-linear} shows that PPI++ has
the same first-order interpretation. Linear calibration can
also be viewed as a special case of TMLE
\citep{hojbjerre2026within}.
\end{remark}

\begin{remark}[Platt-scaling adjustment]
Linear calibration extends beyond squared-error loss to
logistic, Poisson, and other losses associated with
canonical-link generalized linear models
\citep{rosenblum2009using}. For binary outcomes
$Y\in\{0,1\}$ and an initial prediction score
$\mhat(X)\in(0,1)$, one can use Platt scaling
\citep{cox1958two,platt1999probabilistic}, which fits a
logistic regression of $Y$ on the logit-transformed score:
\begin{equation}
\label{eq:platt-cal}
\bigl(\hat a_{\mathrm{platt}},\hat b_{\mathrm{platt}}\bigr)
\in
\arg\min_{a,b\in\R}
\sum_{i=1}^n
\Bigl[
-Y_i \log \sigma\bigl(a\,\mathrm{logit}(\mhat(X_i))+b\bigr)
-(1-Y_i)\log\Bigl(
1-\sigma\bigl(a\,\mathrm{logit}(\mhat(X_i))+b\bigr)
\Bigr)
\Bigr],
\end{equation}
where $\sigma(t):=(1+e^{-t})^{-1}$. The calibrated
predictor is
\[
m_{n,\mathrm{platt}}^\star(X)
:=
\sigma\bigl(
\hat a_{\mathrm{platt}}\,\mathrm{logit}(\mhat(X))
+\hat b_{\mathrm{platt}}
\bigr).
\]
Platt scaling can be useful when predicted probabilities
are systematically over- or under-confident and the
calibration curve is approximately sigmoidal. More
generally, outcomes can be rescaled to $[0,1]$ before
applying this procedure. Calibration properties of
Platt scaling in online settings are studied by
\citet{gupta2023online}.
\end{remark}

\subsection{Cross-fitting and leave-one-out calibration}
\label{sec:crossfit}

Fitting a calibrator to pooled out-of-fold predictions still uses each labeled
outcome in the calibration fit. To hold out both steps, partition the labeled
indices into folds \(I_1,\ldots,I_K\) independently of the observations, and
fit each calibrated score \(q^{(-k)}\) without the pairs indexed by \(I_k\).
Any predictor fitting and tuning used to construct \(q^{(-k)}\) also exclude
\(I_k\). Let \(k(i)\) be the fold containing \(i\) and
\(c_n:=N/(n+N)\). The cross-fitted AIPW estimator is
\begin{equation}
\label{eq:calibration-crossfit}
\widehat\psi_{\mathrm{CF}}
:=\frac1n\sum_{i=1}^n
\left[Y_i+c_n\{\Pu q^{(-k(i))}-q^{(-k(i))}(X_i)\}\right].
\end{equation}
For each \(i\), conditional on the training sample,
\(\E[q^{(-k(i))}(X_i)]=\E[\Pu q^{(-k(i))}]\).
Thus \eqref{eq:calibration-crossfit} is unbiased when these predictions are
integrable and the fitted scores do not use the unlabeled observations.

Leave-one-out calibration takes \(K=n\). Writing \(q_{-i}\) for the score
fitted without \((X_i,Y_i)\), we obtain
\begin{equation}
\label{eq:calibration-loo}
\widehat\psi_{\mathrm{LOO}}
=\frac1n\sum_{i=1}^n
\left[Y_i+c_n\{\Pu q_{-i}-q_{-i}(X_i)\}\right].
\end{equation}
The residual correction is retained because the held-out residuals need not
average to zero. In the simulation, we use
\begin{equation}
\label{eq:loo-se}
\widehat{\mathrm{se}}_{\mathrm{LOO}}^2
=\frac{\widehat{\Var}_L\{Y_i-c_nq_{-i}(X_i)\}}{n}
+\frac{\widehat{\Var}_U\{c_n\bar q(\widetilde X_j)\}}{N},
\qquad \bar q:=\frac1n\sum_{i=1}^n q_{-i},
\end{equation}
where the empirical variances use denominators \(n-1\) and \(N-1\).
The bias, coverage, and interval widths of this procedure are reported in
Appendix~\ref{app:small-n}.

\section{Theory for isotonic-calibrated PPI}
\label{sec:theory}

We study the theoretical properties of isotonic-calibrated PPI,
showing how calibration can improve both prediction accuracy and
the efficiency of semisupervised mean estimation. We establish
asymptotic linearity, asymptotic normality, and valid Wald
inference for $\widehat{\psi}_{\mathrm{iso}}$. We then compare
its asymptotic variance with that of AIPW, PPI, PPI++, and
linear calibration, and give conditions under which it attains
the semiparametric efficiency bound based on either $m(X)$
or the full covariate vector $X$
\citep{bickel1993efficient,van2000asymptotic}.
Finally, we show that the fitted isotonic score is
asymptotically calibrated at the population level.
Replacing the fitted isotonic score by its conditional-mean
oracle yields no additional first-order variance reduction.

\subsection{Asymptotic normality, inference, and efficiency}

For a fixed initial score $\mhat$, define the population isotonic
calibration map
\begin{equation}
\label{eq:pop-iso}
f_0 \in \arg\min_{f\in\mathcal F_{\mathrm{iso}}}
\E_{P_0}\bigl[\{Y-f(\mhat(X))\}^2\bigr],
\qquad
m_0 = f_0 \circ \mhat.
\end{equation}
Thus, $m_0$ is the population $L^2$-optimal monotone
transformation of $\mhat(X)$. Since the identity map belongs
to $\mathcal F_{\mathrm{iso}}$, its mean squared prediction
error is no greater than that of the original score.
Under the following regularity conditions, the isotonic
estimator $m_{n,\mathrm{iso}}^\star$ achieves the classical
$n^{-2/3}$ mean squared error rate
\citep{chatterjee2015risk,yang2018contraction}.

\begin{assumption}[Regularity conditions for isotonic calibration]
\label{assump:isoreg}
Assume:
\begin{enumerate}[leftmargin=6mm,label=(\roman*)]
\item \textbf{Boundedness:} There exists $C_0<\infty$ such that
$|\E_{P_0}[Y \mid \mhat(X)]|\le C_0$ almost surely and
$\sup_x |m_{n,\mathrm{iso}}^\star(x)| = O_p(1)$.
\label{cond::bound}
\item \textbf{Controlled tails:} For \(T=\mhat(X)\), there exists
\(K<\infty\) such that
\[
\E_{P_0}\left[
\left.\exp\left\{\frac{|Y-\E_{P_0}[Y\mid T]|}{K}\right\}\right|T
\right]\le2
\qquad\text{almost surely}.
\]
\label{cond::subgaus}
\end{enumerate}
\end{assumption}

Condition~\ref{cond::bound} requires a bounded conditional
mean and a fitted isotonic regression that remains bounded
in probability. Condition~\ref{cond::subgaus} permits
unbounded outcomes with sub-Gaussian or sub-exponential
errors, covering binary, Gaussian, Poisson, and bounded
outcomes. Results for isotonic regression under heavier-tailed
errors are also available
\citep{han2018robustness,han2019convergence}.

\begin{theorem}[Asymptotic linearity under the two-sample model]
\label{thm:al}
Suppose \cref{assump:design,assump:isoreg} hold. Then
\[
\|m_{n,\mathrm{iso}}^\star-m_0\|_{2,\Px}^2
=
O_p(n^{-2/3}).
\]
Define
\[
D_{m_0,n}^L(X,Y)
:=
Y-\psi_0-(1-\rhon)\{m_0(X)-\psi_0\},
\qquad
D_{m_0,n}^U(\widetilde X)
:=
(1-\rhon)\{m_0(\widetilde X)-\psi_0\}.
\]
Then
\begin{align}
\widehat{\psi}_{\mathrm{iso}}-\psi_0
&=
(\Pl-P_0)D_{m_0,n}^L
+
(\Pu-\Px)D_{m_0,n}^U
+
R_{n,N},
\label{eq:al}\\
R_{n,N}
&=
O_p\bigl((1-\rhon)n^{-2/3}\bigr)
+
O_p\bigl((1-\rhon)n^{-1/3}N^{-1/2}\bigr)
=o_p(n^{-1/2}).
\nonumber
\end{align}
Consequently, $\widehat{\psi}_{\mathrm{iso}}$ is
asymptotically linear at the $\sqrt n$ rate with triangular
influence pair $(D_{m_0,n}^L,D_{m_0,n}^U)$.
\end{theorem}

The asymptotic linear representation yields a mean-zero
normal limit for
$\sqrt n(\widehat{\psi}_{\mathrm{iso}}-\psi_0)$ with variance
\[
\sigma_0^2
:=
\Var\!\left[
Y-\psi_0-(1-\rhozero)\{m_0(X)-\psi_0\}
\right]
+
\rhozero(1-\rhozero)\Var\{m_0(X)\}.
\]
The following corollary establishes asymptotic normality
and valid Wald inference using an empirical estimate of
this variance.

\begin{corollary}[Asymptotic normality and Wald inference]
\label{cor:normal}
Under the conditions of \cref{thm:al}, assume that
$\sigma_0^2>0$. Define the empirical variance estimator
\[
\widehat \sigma^2
:=
\Pl\{(\widehat D^L-\Pl\widehat D^L)^2\}
+
\frac{n}{N}\Pu\{(\widehat D^U-\Pu\widehat D^U)^2\},
\]
where
\[
\widehat D_i^L
:=
Y_i-\widehat{\psi}_{\mathrm{iso}}
-(1-\rhon)\{m_{n,\mathrm{iso}}^\star(X_i)-\widehat{\psi}_{\mathrm{iso}}\},
\qquad
\widehat D_j^U
:=
(1-\rhon)\{m_{n,\mathrm{iso}}^\star(\widetilde X_j)-\widehat{\psi}_{\mathrm{iso}}\}.
\]
Then
\[
\frac{\sqrt n\,(\widehat{\psi}_{\mathrm{iso}}-\psi_0)}
{\widehat\sigma}
\overset{d}{\longrightarrow}
N(0,1).
\]
\end{corollary}

We next characterize the efficiency of isotonic-calibrated
PPI relative to other transformations of the initial score.
The distinction is between efficiency in the full two-sample
model and efficiency when regression adjustment is restricted
to functions of $\mhat(X)$.

\begin{corollary}[Efficiency regimes]
\label{cor:efficiency}
Under the conditions of \cref{thm:al}, the following hold.
\begin{enumerate}[leftmargin=6mm,label=(\roman*)]
\item If $m_0=\muzero$ almost surely, then
$\widehat{\psi}_{\mathrm{iso}}$ is semiparametrically efficient
for $\psi_0$ in the full two-sample model.

\item More generally, $\widehat{\psi}_{\mathrm{iso}}$ is
asymptotically at least as efficient as any AIPW estimator
of the form $\widehat{\psi}(\theta\circ\mhat)$, where
$\theta$ is nondecreasing and $\theta\circ\mhat$ is square-integrable.

\item If $\E_{P_0}[Y\mid \mhat(X)]$ is monotone nondecreasing
in $\mhat(X)$, then
$m_0(X)=\E_{P_0}[Y\mid \mhat(X)]$ almost surely.
In this case, no AIPW estimator based on a square-integrable transformation
$\theta\circ\mhat$, with
$\theta:\mathbb R\to\mathbb R$, has smaller asymptotic
variance than $\widehat{\psi}_{\mathrm{iso}}$.
Equivalently, $\widehat{\psi}_{\mathrm{iso}}$ is
semiparametrically efficient in the reduced two-sample
model obtained by replacing $X$ with $\mhat(X)$.
\end{enumerate}
\end{corollary}

Taking \(\theta\) constant in part (ii) shows that isotonic calibration is
asymptotically at least as efficient as the labeled mean.
At the \(\sqrt n\) scale, the variance reduction is
\[\Var(Y)-\sigma_0^2=(1-\rhozero)\Var\{m_0(X)\}.\]
Corollary~\ref{cor:efficiency} shows that isotonic calibration
is at least as efficient as other monotone transformations,
including AIPW, PPI, PPI++, and linear calibration when
the corresponding scaling or slope coefficient is
nonnegative. The improvement can be strict when the optimal
transformation is nonlinear. If the score is negatively
associated with the outcome, its sign can be reversed before
isotonic calibration. Full semiparametric efficiency is
attained when the calibrated score recovers the conditional
mean $\mu_0(X)$. Inefficiency can arise either because the scalar score loses
information about the conditional mean or because the
monotonicity constraint prevents its recovery from the score.

When \(N/n\to\infty\), the limiting labeled fraction is zero, and the
asymptotic variance reduces to
\[
\sigma_0^2=\Var\{Y-m_0(X)\}.
\]
If the conditional mean of \(Y\) given \(m(X)\) is nondecreasing in the
score, this equals \(\E[\Var\{Y\mid m(X)\}]\). The unlabeled contribution
is negligible at the \(\sqrt n\) scale, and the labeled residuals determine
first-order uncertainty.

\subsection{First-order equivalence to oracle calibration}

We next consider whether further post-processing of the
isotonically calibrated score can improve prediction or
estimation efficiency. Since $\widehat{\psi}_{\mathrm{iso}}$
depends on the covariates only through the fitted score
$m_{n,\mathrm{iso}}^\star(X)$, we consider the reduced
two-sample experiment
\[
\Bigl(
\{(m_{n,\mathrm{iso}}^\star(X_i),Y_i)\}_{i=1}^n,\,
\{m_{n,\mathrm{iso}}^\star(\widetilde X_j)\}_{j=1}^N
\Bigr),
\]
in which the full covariate vector $X$ is replaced by
$m_{n,\mathrm{iso}}^\star(X)$.

The optimal regression adjustment in this reduced experiment
is $\E_{P_0}[Y\mid m_{n,\mathrm{iso}}^\star(X)]$, where the
conditional expectation is evaluated for a new observation from $P_0$, treating
the data used to fit
$m_{n,\mathrm{iso}}^\star$ as fixed. We show that replacing
the isotonic score by this oracle regression yields no
first-order improvement in estimation efficiency.

\begin{assumption}[Bounded-variation score regression]
\label{assump:eff-bv}
The map $t\mapsto \E_{P_0}[Y\mid \mhat(X)=t]$
has bounded variation on the support of $\mhat(X)$.
\end{assumption}

Assumption~\ref{assump:eff-bv} requires the conditional mean
of $Y$ given the initial score to have bounded total variation.
It permits nonmonotone functions and discontinuities, and does not require differentiability
or Lipschitz continuity \citep{van2023causal}.

\begin{theorem}[First-order equivalence to oracle calibration]
\label{prop:eif-red}
Suppose \cref{assump:design,assump:isoreg,assump:eff-bv} hold. Let
\[
\widehat{\psi}_{\mathrm{eff}}
:=
\widehat{\psi}(m_{0\mid n}),
\qquad
m_{0\mid n}(X)
:=
\E_{P_0}[Y \mid m_{n,\mathrm{iso}}^\star(X)]
\]
denote the oracle AIPW estimator in the reduced experiment
based on the fitted isotonic score. Then
\[
\|m_{0\mid n}-m_{n,\mathrm{iso}}^\star\|_{2,\Px}^2
=
O_p(n^{-2/3}),
\]
and
\[
\widehat{\psi}_{\mathrm{eff}}
-
\widehat{\psi}_{\mathrm{iso}}
=
o_p(n^{-1/2}).
\]
In particular, $\widehat{\psi}_{\mathrm{iso}}$ and
$\widehat{\psi}_{\mathrm{eff}}$ have the same asymptotic variance.
\end{theorem}

Theorem~\ref{prop:eif-red} establishes that the fitted isotonic
score is calibrated up to an $O_p(n^{-2/3})$ squared $L^2$
error. Together with the convergence result in
\cref{thm:al} and the population optimality of $m_0$,
this shows that isotonic calibration does not asymptotically
increase mean squared prediction error relative to the
original score. This is calibeating in terms of squared prediction error.

The theorem also establishes first-order equivalence between
$\widehat{\psi}_{\mathrm{iso}}$ and the oracle AIPW estimator
based on $m_{0\mid n}$. Thus, replacing the fitted isotonic score by its
conditional-mean oracle yields no additional first-order
variance reduction. Moreover, by \cref{thm:al,prop:aipw-class},
isotonic calibration achieves the smallest limiting variance
among monotone post-processings of the original score and
is therefore at least as efficient as AIPW and PPI.
This is calibeating in terms of first-order estimation efficiency.

\section{Experiments}
\label{sec:exp}

We compare the proposed calibration methods against the
labeled-only estimator, standard AIPW and PPI, and the
empirical efficiency maximization methods PPI++ and
\textsc{AIPW-EM}. The calibration methods include
\texttt{LinearCal}, \texttt{IsoCal}, and \texttt{AutoCal}.
We first examine the effects of score miscalibration in
controlled simulations, then evaluate the methods on
established benchmarks and an application involving
LLM-generated predictions.
Empirical relative efficiency is the ratio of the reference estimator's
Monte Carlo variance to that of the estimator being evaluated; each caption
specifies the reference estimator.

\subsection{Simulation study}
\label{sec:simulation}

We consider a binary-outcome simulation designed to examine the
effects of score miscalibration. Let $S\sim N(0,1)$, define
\[
\mu_0(S)=\{1+\exp(-5S)\}^{-1},
\]
and generate $Y\mid S\sim\mathrm{Bernoulli}\{\mu_0(S)\}$.
We construct a miscalibrated prediction score by applying
a nonlinear monotone transformation:
\[
m(X)=\mathrm{clip}\left[
-0.15+0.75\,\sigma\bigl\{3\,z(S)-0.5\bigr\},
\,0.01,\,0.99
\right],
\]
where $z(S):=\mathrm{logit}\{\mu_0(S)\}$ and
$\sigma(t):=\{1+\exp(-t)\}^{-1}$.
This transformation preserves the ordering of the conditional
mean, apart from ties induced by clipping, while introducing
substantial miscalibration on the probability scale.

We compare the labeled-only mean (\textsc{Labeled-only}),
PPI, AIPW, PPI++, empirical efficiency-maximized AIPW
(\textsc{AIPW-EM}), linear calibration (\textsc{LinearCal}),
monotone spline calibration (\textsc{MonoSpline})
\citep{ramsay1988monotone,jiang2011smooth}, isotonic calibration
(\textsc{IsoCal}), and an adaptive procedure (\textsc{AutoCal})
that selects among AIPW and the calibration methods by
cross-validated empirical efficiency. For PPI++, we use
the official \texttt{ppi\_py} implementation, which selects
a scaling coefficient by empirical efficiency maximization
and restricts it to $[0,1]$ in the original PPI parameterization,
equivalently $[0,(1-\rho_n)^{-1}]$ in the AIPW parameterization.
By contrast, \textsc{AIPW-EM} minimizes the same criterion
without restricting the coefficient. We vary the labeled
sample size over $n\in\{50,100,200,400,800,1200,2400\}$
and consider two unlabeled-to-labeled sample size ratios,
$N/n\in\{1,16\}$. Results are based on 500 Monte Carlo
repetitions per configuration.

\begin{figure}[htb]
\centering
\includegraphics[width=\textwidth]{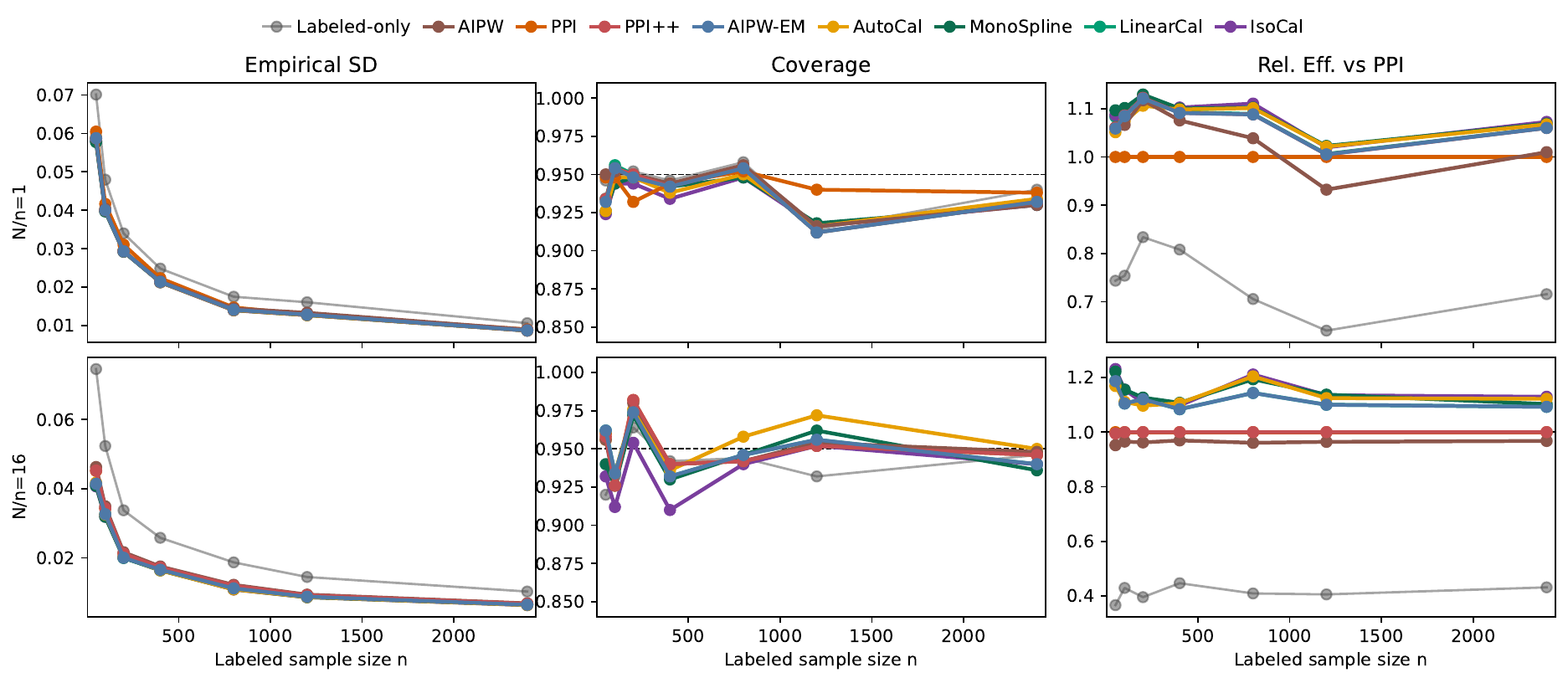}
\caption{Simulation results with a miscalibrated prediction score.
The top and bottom rows correspond to $N/n=1$ and $N/n=16$,
respectively. The columns show Monte Carlo standard deviation,
coverage of nominal 95\% confidence intervals, and relative
efficiency compared with PPI, with values above one indicating
greater efficiency. Bias was negligible relative to standard
errors and is omitted.}
\label{fig:simulation-rmse}
\end{figure}

Figure~\ref{fig:simulation-rmse} shows that calibration improves
efficiency when the prediction score is miscalibrated.
At $n=400$ and $N/n=16$, \texttt{LinearCal}, \texttt{MonoSpline}, and
\texttt{IsoCal} have RMSEs between $0.0164$ and $0.0166$, compared with
$0.0172$ for PPI and $0.0175$ for AIPW. Their nominal 95\% Wald-interval
coverage ranges from $0.910$ to $0.932$.
At $n=1200$, \textsc{AutoCal}, \textsc{MonoSpline}, and
\textsc{IsoCal} achieve RMSEs of approximately $0.0087$,
compared with $0.0092$ for PPI.

The comparison between PPI++ and \textsc{AIPW-EM} illustrates
the effect of restricting the scaling coefficient.
\textsc{AIPW-EM} performs similarly to \textsc{LinearCal},
whereas PPI++ approaches the performance of unadjusted PPI
when the unlabeled sample is large. Averaged across sample
sizes with $N/n=16$, RMSE is approximately $0.0196$ for
\textsc{AIPW-EM} and \textsc{LinearCal}, compared with
$0.0208$ for PPI++. This difference is consistent with the
constraint imposed by the \texttt{ppi\_py} implementation,
which restricts the PPI++ scaling coefficient to $[0,1]$,
equivalently $[0,(1-\rho_n)^{-1}]$ in the AIPW
parameterization. Finally, \textsc{AutoCal} performs
comparably to the best fixed calibration methods, while
\textsc{MonoSpline} achieves similar accuracy to
\textsc{IsoCal}.

\subsection{Small-sample comparisons}
We also examine smaller labeled samples under the same binary-outcome model,
using mild and strong score miscalibration as specified in
Appendix~\ref{app:small-n}. We set \(N=16n\) for
\(n\in\{20,50,100,200,400,800,1200\}\) and use 1,000 replications per setting.
The labeled mean, PPI, AIPW, PPI++, linear, smooth monotone, and isotonic
calibration are evaluated on the same samples, together with the leave-one-out
linear and isotonic estimators in \eqref{eq:calibration-loo}.

Under strong miscalibration, linear, smooth monotone, and isotonic calibration
each have lower RMSE than PPI++ at all seven sample sizes. At \(n=20\),
isotonic coverage ranges from \(0.814\) to \(0.871\) across the two settings;
leave-one-out calibration increases it to \(0.905\)--\(0.925\), with wider
intervals and larger RMSE. Across all 14 settings, leave-one-out calibration
lowers linear-calibration RMSE by at most \(1.99\%\), but lowers isotonic RMSE
in only one setting. The RePPI random-forest and neural-network comparisons
are reported in Appendix~\ref{app:small-n}.

\begin{figure}[htb]
\centering
\includegraphics[width=\textwidth]{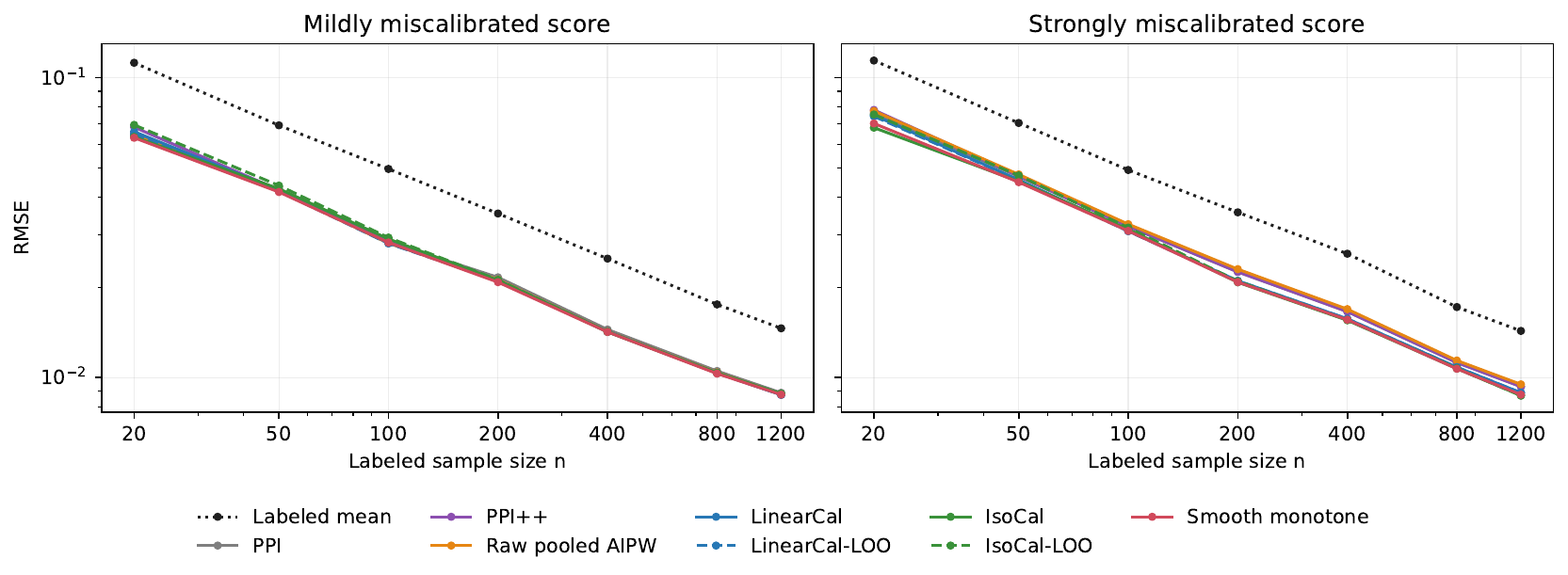}
\caption{RMSE across labeled sample sizes under mild and strong score
miscalibration, with \(N=16n\) and 1,000 replications per setting. Solid and
dashed curves distinguish calibration on the labeled sample from
leave-one-out calibration.}
\label{fig:small-n-rmse}
\end{figure}

\subsection{Empirical benchmarks}
\label{sec:empirical}

We evaluate the proposed estimators on the \texttt{forest},
\texttt{galaxies}, and \texttt{census\_income} benchmarks used by
\citet{ppi2023,angelopoulos2023ppi} and distributed with the
\texttt{ppi\_py} package. Following the original labeled-sample-size
grids, we repeatedly partition each fully labeled dataset into
labeled and unlabeled samples, evaluate all estimators on the same
split, and use the full-sample mean as the reference value.
Results are based on 500 random splits per configuration.\footnote{
Coverage is evaluated over repeated partitions of a fixed dataset,
rather than independent draws from a superpopulation. Consequently,
coverage of Wald confidence intervals need not equal the nominal
level.}
We compare PPI, AIPW, PPI++, \textsc{AIPW-EM},
\texttt{LinearCal}, and \texttt{AutoCal}. Additional calibration
methods and dataset-specific results are reported in
Appendix~\ref{app:empirical}. In these benchmarks and the LLM experiments,
\texttt{LinearCal} clips affine predictions to the labeled outcome range,
and \texttt{IsoCal} first groups observations into bins of at least 10
before fitting isotonic regression.

\begin{figure}[htb]
\centering
\includegraphics[width=\textwidth]{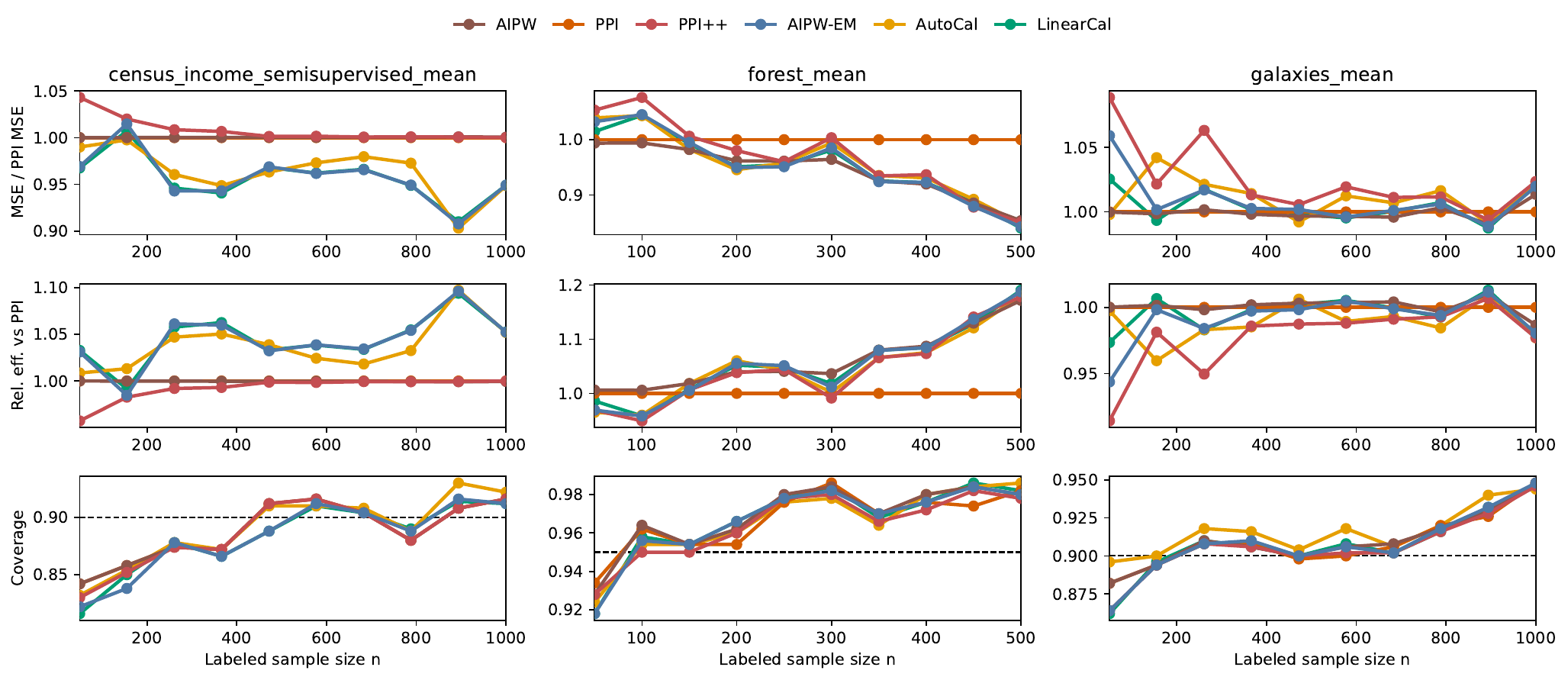}
\caption{Results on the \texttt{ppi\_py} benchmarks.
Panels show MSE normalized by that of PPI, relative efficiency
compared with PPI, and confidence-interval coverage.
Nominal coverage is 95\% for \texttt{forest} and 90\% for
\texttt{galaxies} and \texttt{census\_income}.}
\label{fig:ppi-main}
\end{figure}

Figure~\ref{fig:ppi-main} shows that AIPW generally improves
on PPI, particularly when the unlabeled sample is not much
larger than the labeled sample. Calibration provides further
gains, most notably on \texttt{census\_income}, where
\texttt{LinearCal} achieves the lowest average MSE.
The isotonic and monotone spline methods perform similarly
to the leading estimators
(\Cref{fig:ppi-calibration-appendix}).

The comparison with PPI++ also shows the effect of restricting
the empirical efficiency-maximizing coefficient.
Across the 30 dataset-size configurations,
\textsc{AIPW-EM} achieves lower MSE than PPI++ in 29
and performs nearly identically to \texttt{LinearCal}.
All three methods have similar empirical coverage,
averaging approximately 92\% across configurations with
nominal coverage of either 90\% or 95\%. Thus, much of the improvement over PPI++
is attributable to removing the constraint on its scaling
coefficient, rather than to nonlinear calibration.
\texttt{AutoCal} also performs comparably to the best
fixed method across the benchmarks.

\subsection{LLM evaluation benchmark}
\label{sec:llm-benchmark}

We evaluate the methods on two public LLM benchmarks. The first is
PPE Human Preference V1 \citep{frick2025evaluate}, where the
estimand for each target model is its human win rate across
prompt--response comparisons, counting ties as half-wins.
The second is a PPE correctness benchmark \citep{frick2025evaluate}
constructed from MMLU-Pro \citep{wang2024mmlupro},
MATH \citep{hendrycksmath2021}, GPQA \citep{rein2024gpqa},
IFEval \citep{zhou2023instructionfollowing}, and MBPP+
\citep{liu2023evalplus}. Each observation consists of a sampled pair containing one correct
and one incorrect response from the same model, and the
estimand is the probability that the first response is correct.
We use the public model outputs, labels, and evaluator scores,
randomizing only the labeled and unlabeled split. Results are
based on 200 repeated splits. We compare labeled-only estimation,
PPI, AIPW, PPI++, \textsc{AIPW-EM}, and the proposed calibration
methods.

\IfFileExists{assets/fig_llm_eval_main.pdf}{
\begin{figure}[htb]
\centering
\includegraphics[width=\textwidth]{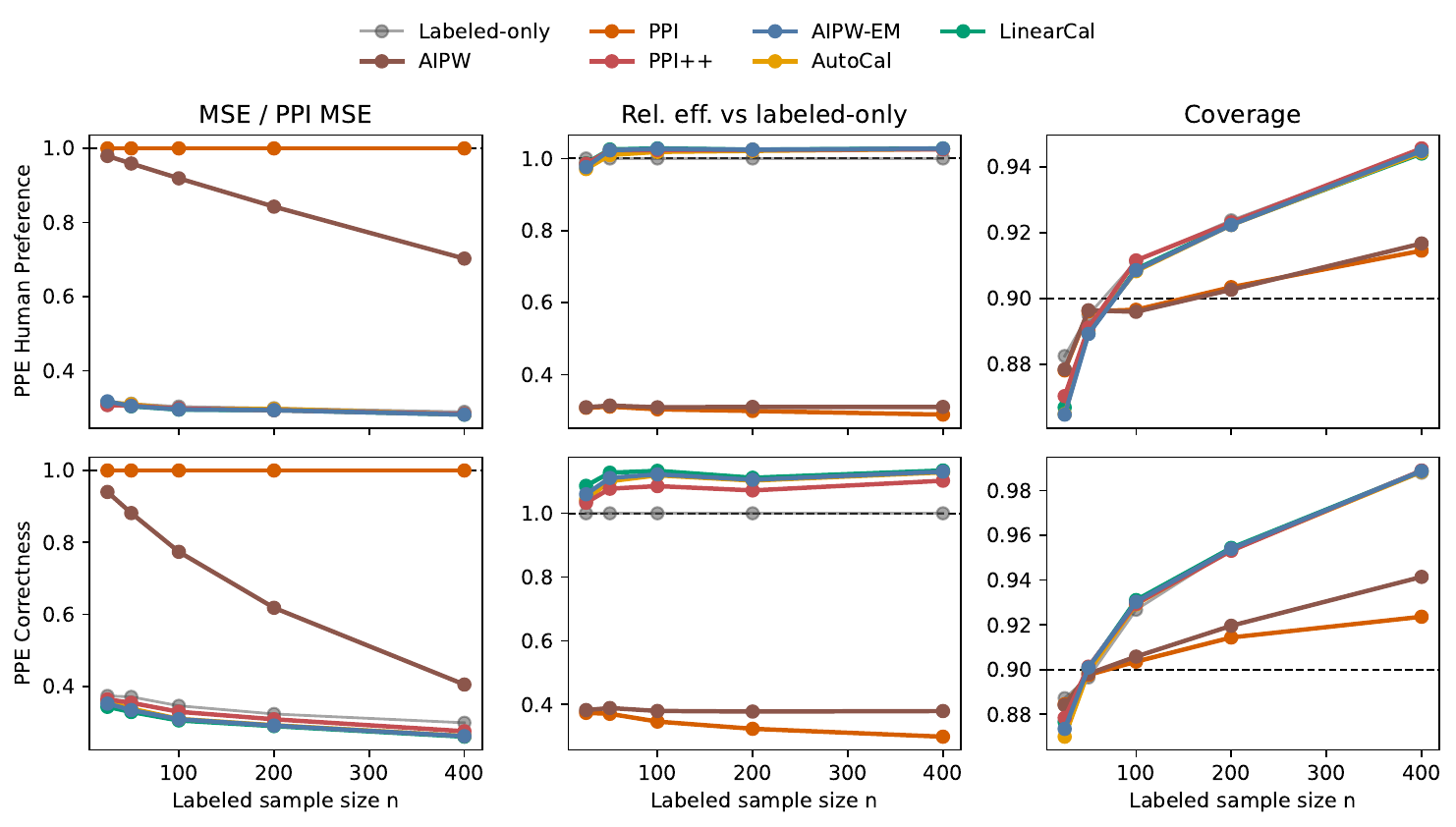}
\caption{LLM evaluation results on PPE Human Preference and
PPE Correctness. Panels show MSE normalized by that of PPI,
relative efficiency compared with labeled-only estimation,
and coverage of nominal 90\% confidence intervals.
Results are averaged across the public evaluator models.}
\label{fig:llm-eval-main}
\end{figure}
}{}

For each benchmark, evaluator, and target model, MSE/PPI is the estimator's
Monte Carlo MSE divided by that of PPI. Implied label savings are
\[
\max\!\left\{0,\,
1-\frac{\widehat{\Var}(\widehat\psi)}{\widehat{\Var}(\overline Y_L)}
\right\},
\]
where the variances are computed over repeated splits and \(\overline Y_L\)
is the labeled mean. We set negative savings to zero before averaging.
Table entries are averaged first over target models within each benchmark,
then over benchmarks within each evaluator, and then over evaluators.

\IfFileExists{assets/table_llm_summary.tex}{
\begin{table}[htb]
\centering
\small
\resizebox{\textwidth}{!}{\begin{tabular}{lllccc}
\toprule
Track & $n$ & Estimator & MSE / PPI & Label savings & Coverage \\
\midrule
PPE Correctness & 100 & Labeled-only & 0.346 & 0.000 & 0.927 \\
PPE Correctness & 100 & PPI & 1.000 & 0.010 & 0.904 \\
PPE Correctness & 100 & AIPW & 0.774 & 0.009 & 0.906 \\
PPE Correctness & 100 & LinearCal & 0.305 & 0.101 & 0.931 \\
PPE Correctness & 100 & AutoCal & 0.310 & 0.094 & 0.929 \\
PPE Correctness & 400 & Labeled-only & 0.299 & 0.000 & 0.988 \\
PPE Correctness & 400 & PPI & 1.000 & 0.020 & 0.924 \\
PPE Correctness & 400 & AIPW & 0.405 & 0.009 & 0.941 \\
PPE Correctness & 400 & LinearCal & 0.260 & 0.105 & 0.989 \\
PPE Correctness & 400 & AutoCal & 0.263 & 0.102 & 0.988 \\
PPE Human Preference & 100 & Labeled-only & 0.303 & 0.000 & 0.911 \\
PPE Human Preference & 100 & PPI & 1.000 & 0.004 & 0.897 \\
PPE Human Preference & 100 & AIPW & 0.919 & 0.004 & 0.896 \\
PPE Human Preference & 100 & LinearCal & 0.295 & 0.031 & 0.909 \\
PPE Human Preference & 100 & AutoCal & 0.297 & 0.028 & 0.908 \\
PPE Human Preference & 400 & Labeled-only & 0.289 & 0.000 & 0.945 \\
PPE Human Preference & 400 & PPI & 1.000 & 0.005 & 0.915 \\
PPE Human Preference & 400 & AIPW & 0.703 & 0.004 & 0.917 \\
PPE Human Preference & 400 & LinearCal & 0.282 & 0.028 & 0.944 \\
PPE Human Preference & 400 & AutoCal & 0.282 & 0.027 & 0.944 \\
\bottomrule
\end{tabular}
}
\caption{LLM evaluation results at $n\in\{100,400\}$.
Reported are MSE normalized by that of PPI, implied label
savings relative to labeled-only estimation, and coverage of nominal 90\% Wald intervals.
Results for all estimators are reported in
Appendix~\ref{tab:llm-eval-summary-full}.}
\label{tab:llm-eval-summary}
\end{table}
}{}

Raw evaluator scores can substantially increase estimation
variance. On PPE Human at $n=100$, PPI and AIPW have MSEs
of $0.495$ and $0.440$, respectively, compared with
$0.001425$ for labeled-only estimation. The corresponding
MSEs on PPE Correctness are $0.534$, $0.353$, and $0.002010$.
Although AIPW improves on PPI, both perform considerably
worse than ignoring the evaluator scores. The largest losses
are associated with evaluators such as Skywork and Athene,
whereas ArmoRM exhibits substantially less miscalibration
(\Cref{fig:llm-eval-by-evaluator}).

Calibration largely eliminates these efficiency losses.
On PPE Human, \texttt{LinearCal} has MSE $0.001391$ at
$n=100$ and $0.000278$ at $n=400$, compared with
$0.001425$ and $0.000286$ for labeled-only estimation.
Relative to labeled-only estimation, these are MSE reductions
of $2.4\%$ and $2.6\%$, respectively. The gains are larger on PPE Correctness,
where \texttt{LinearCal} reduces MSE from $0.002010$ to
$0.001813$ at $n=100$ and from $0.000248$ to $0.000222$
at $n=400$, corresponding to MSE reductions of $9.8\%$ and $10.5\%$,
respectively. At $n=100$, its nominal 90\% Wald intervals have empirical
coverage of $0.909$ on PPE Human and $0.931$ on PPE Correctness.
\texttt{AutoCal} performs similarly, while
\textsc{MonoSpline} provides gains over raw PPI and AIPW
and remains competitive with PPI++ and \textsc{AIPW-EM}.
The relatively strong performance of linear calibration
suggests that correcting the scale of the evaluator scores
accounts for much of the available efficiency gain.

Calibration also improves model ranking on PPE Human.
Raw PPI and AIPW recover the model ordering poorly,
whereas calibrated and efficiency-maximized estimators
perform substantially better. At $n=200$, \texttt{AutoCal}
achieves the highest average Spearman correlation of $0.905$.
At $n=400$, \texttt{LinearCal} achieves a Spearman correlation
of $0.941$ and identifies the top-ranked model in $64.6\%$
of repetitions; see Appendix~\ref{app:llm-ranking}.
For MATH scores from ArmoRM, empirical correctness increases across twelve
score bins, from $0.078$ in the lowest bin to $0.920$ in the highest.
Platt calibration fitted on 100 labels reduces prediction MSE from $0.4683$
to $0.2020$ on the held-out observations; see Appendix~\ref{app:llm-ordering}.

\section{Conclusion  }

We introduced \emph{calibrated prediction-powered inference},
a framework for semisupervised mean estimation based on three
steps: obtain a prediction score, calibrate it using the labeled
sample, and average the calibrated predictions over the pooled
covariate sample. This plug-in procedure fits naturally within
the broader literature on semiparametric debiasing and can be
viewed as a semisupervised specialization of calibrated debiased
machine learning \citep{van2024automatic}.

Our results show that calibration can improve semisupervised
efficiency without retraining the original model, and that
isotonic calibration can \emph{calibeat} the original score
and simpler post-processing rules at first order. Moreover,
replacing the fitted isotonic score by its conditional-mean
oracle yields no additional first-order variance reduction. Empirically, when the original score is
already well calibrated, post-processing offers little
improvement, whereas under substantial miscalibration,
even simple linear or monotone calibration can improve
efficiency while retaining a simple plug-in form.

 More broadly, our results place prediction-powered inference
within the established literature on semiparametric estimation,
debiased machine learning, and covariate adjustment in randomized
trials and missing-at-random settings. In particular, we show
that PPI++ is an AIPW estimator based on empirical efficiency
maximization \citep{rubin2008empirical} and, when unconstrained, is first-order
equivalent to linear calibration and classical prognostic-score
regression adjustment
\citep{hansen2008prognostic,lin2013agnostic,schuler2022increasing}.
It is therefore closely related to randomized-trial methods
that learn prognostic scores from larger historical,
external-control, or auxiliary datasets and use them for
regression adjustment in smaller trials
\citep{hansen2008prognostic,moore2009covariate,rosenblum2009using,
benkeser2021improving,hojbjerre2025powering}.
This connection suggests that semisupervised inference could
benefit from closer integration with the missing-data,
causal-inference, and debiased machine-learning literatures.
These areas address closely related problems of using flexible
regression adjustment learned from auxiliary data while
retaining valid inference for the target parameter.
Further progress may come from adapting methods developed
for randomized trials and missing-at-random settings to
semisupervised estimation
\citep{petersen2014causal,dang2023causal,ho2023current,
smith2023application,ho2024examples}.

\textbf{Beyond mean estimation.}
These connections extend beyond the semisupervised mean-estimation
setting considered here. When outcomes or covariates are missing
depending on observed variables, or in settings involving
right-censored time-to-event outcomes and longitudinal data,
existing semiparametric methods provide a natural starting
point for extending calibrated estimation
\citep{robins1994estimation,vanderlaanunified,bang2005doubly,
moore2009increasing,rose2011targeted,van2018targeted}.
More generally, M-estimation and estimating-equation methods
for missing-at-random problems are well developed
\citep{robins1994estimation,vanderlaanunified,DoubleML}.
When missingness depends on unobserved outcomes even after
conditioning on observed data, however, naively combining
labeled and unlabeled data can introduce bias.
Related work addresses the combination of gold-standard
and potentially biased auxiliary data using semiparametric
and machine-learning methods
\citep{kallus2018removing,rosenman2025methods,
dang2025experiment,van2026adaptive}.
When the unlabeled sample is selected adaptively,
the problem is also related to two-phase sampling,
informative sampling, and adaptive data collection
\citep{van2008construction,chow2008adaptive,rose2011targeted,
malenica2021adaptive,zrnic2024active,zhang2025efficient,
liu2025robust}. Extending calibrated inference to these
missing-data and adaptive-sampling settings is a natural
direction for future work.

  While we focus on calibration of regression functions and inference on means, the ideas here may extend to broader supervised learning problems defined through loss minimization, such as median and quantile regression \citep{rosenblum2009using,noarov2023statistical,jung2021moment,roth2022uncertain,whitehouse2024orthogonal}. They also extend to continuous linear functionals of a regression function or M-estimand through calibrated DML \citep{van2024automatic,van2025automatic}. More generally, when the target parameter is defined by a moment equation, as in M-estimation, one could debias or calibrate the associated augmented moment equation \citep{vanderlaanunified,DoubleML}; see \citet{ji2025predictions}. A practical challenge in this setting is that the augmented loss may be nonconvex or may require separate debiasing at each parameter value, which can complicate optimization. In such cases, a more practical alternative may be a multiaccuracy-style adjustment rather than explicit augmentation or calibration, as in the Efficient Plug-in Learning framework of \citet{van2024combining}, which addresses this issue for nonconvex Neyman-orthogonal loss functions \citep{foster2023orthogonal}.

\bibliographystyle{plainnat}
\bibliography{references}

\appendix

\tableofcontents

\section{Semiparametric background}
\label{section::plugin}

In the unrestricted two-sample model, setting \(f=\mu_0\) in
Proposition~\ref{prop:aipw-class} gives the efficient influence pair
\citep{robins1995analysis,robins1995semiparametric,tsiatis2006semiparametric}.
Regression adjustment and TMLE can yield the same first-order representation
by updating the regression so that the empirical residual mean vanishes
\citep{scharfstein1999adjusting,van2006targeted,rosenblum2009using}.
The calibration identity \eqref{eqn::aipwcal} gives an exact instance of this construction.

\section{Additional remarks}

\label{appendix:remarks}

\begin{remark}[Empirical efficiency maximization over a prediction class]
The empirical variance criterion \(\widehat\sigma_n^2(f)\) in
Section~\ref{sec:effreview} can be minimized over any candidate class
\(\mathcal G\subset L^2(\Px)\), rather than only over rescalings of one score:
\[
\hat f\in\arg\min_{f\in\mathcal G}\widehat\sigma_n^2(f),
\qquad
\widehat\psi_{\mathrm{eem}}:=\widehat\psi(\hat f).
\]
For a fixed score, Proposition~\ref{prop:aipw-class} gives the population criterion
\[
V(f)=\rhozero\Var[Y]+(1-\rhozero)\Var[Y-f(X)].
\]
Writing \(\tilde f=f-\E_{\Px}f+\psi_0\), we have
\(\Var[Y-f(X)]=\E_{P_0}[(Y-\tilde f(X))^2]\).
If \(\mathcal G\) is closed under constant shifts, population least-squares
regression therefore targets a variance-minimizing score.
For tuning, average held-out estimates of the influence-function variance
across folds and use the selected class or tuning parameter in a cross-fitted
AIPW estimator \citep{rubin2008empirical}.
\end{remark}

\begin{remark}[Shrinkage via Venn--Abers]
When the labeled sample is small, isotonic calibration may be unstable. A natural alternative is Venn--Abers calibration, originally introduced for binary classification \citep{vovk2012venn}; see \cite{van2024self, van2025venn} for generalized versions beyond the binary setting. Assume \(Y \in [0,1]\), or rescale otherwise. For each \(x\), augment the labeled calibration sample with the hypothetical point \((\mhat(x),y)\) for \(y \in \{0,1\}\), fit the resulting isotonic calibrators \(f_n^{(x,0)}\) and \(f_n^{(x,1)}\), and form the interval-valued prediction
\[
\bigl[f_n^{(x,0)}\{\mhat(x)\},\, f_n^{(x,1)}\{\mhat(x)\}\bigr].
\]
For the Venn--Abers comparator, we shrink the interval midpoint toward the unadjusted AIPW estimator:
\[
m_{n,\mathrm{VA}}^\star(x)
:=
m_{n,\mathrm{mid}}^\star(x)
+
\bigl\{f_n^{(x,1)}\{\mhat(x)\}-f_n^{(x,0)}\{\mhat(x)\}\bigr\}
\bigl\{\widehat\psi_{AIPW}-m_{n,\mathrm{mid}}^\star(x)\bigr\},
\]
where
\[
\begin{aligned}
m_{n,\mathrm{mid}}^\star(x)
&:=
\frac{f_n^{(x,1)}\{\mhat(x)\}+f_n^{(x,0)}\{\mhat(x)\}}{2}.
\end{aligned}
\]
The shrinkage weight is the width of the Venn--Abers interval.
\end{remark}

\begin{remark}[Intercept-only calibration]
Set \(\hat a_{\mathrm{const}}:=\Pl\{Y-\mhat(X)\}\) and
\(m_{n,\mathrm{const}}^\star:=\mhat+\hat a_{\mathrm{const}}\).
Then PPI is the unlabeled mean \(\Pu m_{n,\mathrm{const}}^\star\),
whereas AIPW is the pooled mean
\(\rhon\Pl m_{n,\mathrm{const}}^\star+(1-\rhon)\Pu m_{n,\mathrm{const}}^\star\).
\end{remark}

\section{Calibration in randomized trials}
\label{app:causal}

Let \(A\in\{0,1\}\) denote randomized treatment and
\(Y=AY(1)+(1-A)Y(0)\). Assume consistency, no interference, and
\(A\perp(X,Y(1),Y(0))\), with positive assignment probabilities.
When estimating \(\mu_a:=\E[Y(a)]\), outcomes from arm \(a\) are labeled,
and covariates from the other arm form the unlabeled sample.
Let \(\mathbb P_{n,a}\) denote the empirical mean in arm \(a\),
\(\mathbb P_n\) the pooled mean, and \(\hat\pi_a=n_a/n\).
If the calibrated arm-specific score satisfies
\(\mathbb P_{n,a}\{Y-m_{n,a}^\star(X)\}=0\), then
\begin{align}
\hat\mu_a^{\mathrm{cal}}
&:=\mathbb P_n\{m_{n,a}^\star(X)\}\label{eq:rct-mua-plugin}\\
&=\mathbb P_n\{m_{n,a}^\star(X)\}
+\mathbb P_n\left[
\frac{\1(A=a)}{\hat\pi_a}\{Y-m_{n,a}^\star(X)\}
\right].\label{eq:rct-mua-aipw}
\end{align}
The residual correction is zero because it equals
\(\mathbb P_{n,a}\{Y-m_{n,a}^\star(X)\}\).
Thus \eqref{eqn::aipwcal} applies to each arm, and the calibrated
treatment-effect estimator is
\begin{equation}
\label{eq:rct-ate}
\hat\tau^{\mathrm{cal}}:=\hat\mu_1^{\mathrm{cal}}-\hat\mu_0^{\mathrm{cal}}.
\end{equation}
These arm-specific updates are instances of calibrated DML
\citep{van2024automatic}.

\section{AIPW comparison and balancing representations}
\label{app:repr}

\begin{lemma}[AIPW comparison identity]
\label{lem:aipw-comparison}
Suppose \(\Pl\{Y-m_n^\star(X)\}=0\).
Then, for any comparison score \(g\) for which the empirical averages are finite,
\[
\psihat
=
\widehat\psi(g)
+(1-\rhon)(\Pu-\Pl)(m_n^\star-g).
\]
\end{lemma}

\begin{proof}
By \eqref{eqn::aipwcal}, \(\psihat=\widehat\psi(m_n^\star)\).
Subtracting \(\widehat\psi(g)\) and collecting terms gives the stated identity.
\end{proof}

\begin{corollary}[Balancing-weight representation]
\label{thm:repr2}
Let \(\mathcal F\) contain the identity map and suppose
\(\Pl[f(\mtilde)\{Y-\mtilde\}]=0\) for every \(f\in\mathcal F\).
If \(\hat w\in\mathcal F\) satisfies
\[
\Pl[\hat w(\mtilde)f(\mtilde)]
=\rhon\Pl f(\mtilde)+(1-\rhon)\Pu f(\mtilde),
\qquad f\in\mathcal F,
\]
then \(\psihat=\Pl[\hat w(\mtilde)Y]\).
\end{corollary}
\begin{proof}
Taking \(f(t)=t\) in the balance equation gives
\(\Pl[\hat w(\mtilde)\mtilde]=\psihat\).
Residual orthogonality with \(f=\hat w\) gives
\(\Pl[\hat w(\mtilde)(Y-\mtilde)]=0\), proving the identity.
\end{proof}

\section{PPI++ and linear calibration}
\label{sec:ppi++}

\begin{proof}[Proof of Proposition~\ref{prop:ppi++-linear}]
Write
\[
\hat v_L:=\Pl[(m-\Pl m)^2],\quad
\hat v_U:=\Pu[(m-\Pu m)^2],\quad
\hat c_L:=\Pl[(Y-\Pl Y)(m-\Pl m)].
\]
On the event \(\hat v_L>0\), the least-squares normal equations give
\[
\hat a_{\mathrm{lin}}=\frac{\hat c_L}{\hat v_L},
\qquad
\hat b_{\mathrm{lin}}=\Pl Y-\hat a_{\mathrm{lin}}\Pl m.
\]
Differentiating the empirical variance criterion in Section~\ref{sec:effreview} gives
\[
\hat\lambda_{++}
=\frac{\hat c_L}{(1-\rhon)\hat v_L+\rhon\hat v_U}.
\]
Finite second moments and \(\Var[m(X)]>0\) imply
\(\hat v_L,\hat v_U\to_p\Var[m(X)]\) and
\(\hat c_L\to_p\Cov(Y,m(X))\).
Thus \(\Pr(\hat v_L>0)\to1\), both coefficients converge to
\(a_0:=\Cov(Y,m(X))/\Var[m(X)]\), and
\(\hat\lambda_{++}-\hat a_{\mathrm{lin}}=o_p(1)\).

For \(\Delta_n:=\Pu m-\Pl m\), substituting the fitted intercept gives
\[
\widehat\psi_{\mathrm{lin}}=\Pl Y+(1-\rhon)\hat a_{\mathrm{lin}}\Delta_n,
\qquad
\widehat\psi_{++}=\Pl Y+(1-\rhon)\hat\lambda_{++}\Delta_n.
\]
Subtracting proves the exact identity in the proposition.
Independence and finite second moments give
\(\Delta_n=O_p(n^{-1/2}+N^{-1/2})\).
Since \(\rhon\to\rhozero<1\), the ratio \(n/N\) is bounded, so
\[
\widehat\psi_{++}-\widehat\psi_{\mathrm{lin}}
=(1-\rhon)o_p(1)\Delta_n=o_p(n^{-1/2}).
\]
\end{proof}

\section{Proof of Proposition~\ref{prop:aipw-class}}
\label{app:eif}

\begin{proof}[Proof of Proposition~\ref{prop:aipw-class}]
Adding a constant to \(f\) leaves \(\widehat\psi(f)\) unchanged.
Thus \(\widehat\psi(f)=\widehat\psi(\tilde f)\), where
\(\tilde f=f-\E_{\Px}f+\psi_0\).
Centering the empirical means gives the exact identity
\[
\widehat\psi(f)-\psi_0
=(\Pl-P_0)D_{f,n}^L+(\Pu-\Px)D_{f,n}^U.
\]
Both components have mean zero.
Write \(A=\mu_0-\psi_0\), \(\delta=\tilde f-\mu_0\),
and \(\varepsilon=Y-\mu_0\). Then
\[
D_{f,n}^L=\varepsilon+\rhon A-(1-\rhon)\delta,
\qquad
D_{f,n}^U=(1-\rhon)(A+\delta).
\]
Since \(\E[\varepsilon\mid X]=0\) and
\(n/N=\rhon/(1-\rhon)\), the covariance terms cancel in
\begin{align*}
\Var(D_{f,n}^L)+\frac nN\Var(D_{f,n}^U)
&=\E[\Var(Y\mid X)]+\rhon\Var(\mu_0)
+(1-\rhon)\Var(\delta).
\end{align*}
Taking limits gives the stated variance, since
\(\Var(\delta)=\Var(f-\mu_0)\).
The choice \(f=\mu_0\) removes the last term and gives the
missing-data efficiency bound
\citep{robins1995analysis,robins1995semiparametric,tsiatis2006semiparametric}.
\end{proof}

\section{Supporting lemmas for isotonic calibration}
\label{app:iso-lemmas}

\begin{lemma}[Blockwise residual orthogonality]
\label{lem:block}
Let \(B_1,\dots,B_J\) be the pooled adjacent violator blocks of the isotonic fit, and let \(m_{n,\mathrm{iso}}^\star(X_i)=c_j\) for \(i \in B_j\). Then
\[
\sum_{i \in B_j}(Y_i-c_j)=0,
\qquad j=1,\dots,J.
\]
Consequently, for every measurable function \(h:\R \to \R\) such that \(h(m_{n,\mathrm{iso}}^\star(X_i))\) is finite for all \(i\),
\[
\sum_{i=1}^n h(m_{n,\mathrm{iso}}^\star(X_i))\{Y_i-m_{n,\mathrm{iso}}^\star(X_i)\}=0.
\]
\end{lemma}

\begin{proof}
The fitted values from isotonic regression are constant on pooled adjacent violator blocks. The KKT conditions imply that residuals sum to zero on each block. Since \(h(m_{n,\mathrm{iso}}^\star(X_i))=h(c_j)\) for all \(i \in B_j\),
\[
\sum_{i=1}^n h(m_{n,\mathrm{iso}}^\star(X_i))\{Y_i-m_{n,\mathrm{iso}}^\star(X_i)\}
=
\sum_{j=1}^J h(c_j)\sum_{i \in B_j}(Y_i-c_j)
=
0.
\]
\end{proof}

\begin{proof}[Proof of \cref{prop:iso-repr}]
Lemma~\ref{lem:block} gives \eqref{eq:orth}.
By definition, the fitted isotonic score minimizes the empirical squared loss
among all nondecreasing transformations of \(\mhat\), proving the risk
comparison. Taking the identity transformation gives the comparison with the
raw score.
\end{proof}

\begin{lemma}[Bounded variation after isotonic post-processing]
\label{lem:bv-post}
Suppose \cref{assump:eff-bv} holds. Let \(T:=\mhat(X)\) and \(S_n:=m_{n,\mathrm{iso}}^\star(X)\). Define
\[
\eta(t):=\E_{P_0}[Y\mid T=t],
\qquad
\eta_n(s):=\E_{P_0}[Y\mid S_n=s].
\]
If \(t\mapsto \eta(t)\) has total variation at most \(V\), then, conditional on the labeled sample, the map \(s\mapsto \eta_n(s)\) has total variation at most \(V\). If, in addition, \(|\eta(t)|\le C_0\) on the support of \(T\), then \(|\eta_n(s)|\le C_0\) on the support of \(S_n\).
\end{lemma}
\begin{proof}
We adapt the proof of Lemma 6 in \cite{van2023causal}. Let \(S_n=f_n(T)\), where \(f_n\) is the fitted isotonic map. Conditional on the labeled sample, \(f_n\) is nondecreasing. Hence, for each \(s\) in the support of \(S_n\), there exists an interval \(B_s\subset \mathbb R\), possibly a singleton, such that
\[
\{S_n=s\}=\{T\in B_s\}.
\]
Therefore
\[
\eta_n(s)
=
\E_{P_0}[Y\mid S_n=s]
=
\E_{P_0}[\eta(T)\mid T\in B_s].
\]

Let \(\eta=\eta^+-\eta^-\) be the Jordan decomposition of \(\eta\), where \(\eta^+\) and \(\eta^-\) are nondecreasing and
\[
\|\eta\|_{TV}=\|\eta^+\|_{TV}+\|\eta^-\|_{TV}\le V.
\]
Then
\[
\eta_n(s)
=
\E_{P_0}[\eta^+(T)\mid T\in B_s]
-
\E_{P_0}[\eta^-(T)\mid T\in B_s].
\]

If \(s_1<s_2\), then monotonicity of \(f_n\) implies every point of \(B_{s_1}\) is no larger than every point of \(B_{s_2}\). Hence the conditional distribution of \(T\mid T\in B_{s_2}\) first-order stochastically dominates that of \(T\mid T\in B_{s_1}\). Therefore, for any nondecreasing function \(g\),
\[
\E_{P_0}[g(T)\mid T\in B_{s_1}]
\le
\E_{P_0}[g(T)\mid T\in B_{s_2}].
\]
Applying this with \(g=\eta^+\) and \(g=\eta^-\), both maps
\[
s\mapsto \E_{P_0}[\eta^+(T)\mid T\in B_s]
\qquad\text{and}\qquad
s\mapsto \E_{P_0}[\eta^-(T)\mid T\in B_s]
\]
are nondecreasing.

Moreover, each of these conditional expectations takes values in the range of the corresponding function, so
\[
\left\|s\mapsto \E_{P_0}[\eta^+(T)\mid T\in B_s]\right\|_{TV}
\le
\|\eta^+\|_{TV},
\qquad
\left\|s\mapsto \E_{P_0}[\eta^-(T)\mid T\in B_s]\right\|_{TV}
\le
\|\eta^-\|_{TV}.
\]
Hence
\[
\|\eta_n\|_{TV}
\le
\|\eta^+\|_{TV}+\|\eta^-\|_{TV}
=
\|\eta\|_{TV}
\le V.
\]

Finally, if \(|\eta(t)|\le C_0\) on the support of \(T\), then
\[
|\eta_n(s)|
=
\left|\E_{P_0}[\eta(T)\mid T\in B_s]\right|
\le C_0
\]
for all \(s\) in the support of \(S_n\).
\end{proof}

\begin{lemma}[The population isotonic projection is centered]
\label{lem:center}
Under \cref{assump:design,assump:isoreg},
\(\E_{\Px}[m_0(X)]=\psi_0\).
\end{lemma}
\begin{proof}
The isotonic class is closed under constant shifts, so
\(c\mapsto\E[(Y-m_0(X)-c)^2]\) is minimized at \(c=0\).
Its derivative there is \(-2\E[Y-m_0(X)]\), giving the stated identity.
\end{proof}

\textbf{Proof convention.} In the following proofs, fix \(\zeta>0\). By \(\sup_x |m_{n,\mathrm{iso}}^\star(x)|=O_p(1)\), there exist a deterministic constant \(C=C_\zeta\ge C_0\) and \(n_0<\infty\) such that the event
\[
A_{n,C}:=\Bigl\{\sup_x |m_{n,\mathrm{iso}}^\star(x)|\le C\Bigr\}
\]
satisfies \(P_0(A_{n,C})\ge 1-\zeta\) for all \(n\ge n_0\). We carry out the deterministic bounded-class arguments on \(A_{n,C}\), where all implicit constants may depend on \(C\) but not on \(n\). Since \(\zeta\) is arbitrary, the resulting rates and distributional statements hold unconditionally.

\begin{lemma}[Isotonic \(L^2\) rate]
\label{lem:isorate}
Suppose \cref{assump:design,assump:isoreg} holds. Then
\[
\|m_{n,\mathrm{iso}}^\star-m_0\|_{2,\Px}^2=O_p(n^{-2/3}).
\]
\end{lemma}

\begin{proof}[Proof of \cref{lem:isorate}]

Let \(T:=\mhat(X)\), \(\eta(T):=\E_{P_0}[Y\mid T]\), and
\[
\mathcal M_C:=\{f(T):f\in\mathcal F_{\mathrm{iso}},\ \|f(T)\|_\infty\le C\}.
\]
By \cref{assump:isoreg}, \(\eta(T)\) is bounded by \(C_0\le C\), and on \(A_{n,C}\) we have \(m_{n,\mathrm{iso}}^\star\in\mathcal M_C\). Since clipping any monotone candidate to \([-C,C]\) preserves monotonicity and can only decrease squared risk against \(\eta(T)\), the population isotonic projection \(m_0\) also belongs to \(\mathcal M_C\). For every \(m\in\mathcal M_C\),
\[
\E_{P_0}[(Y-m)^2]
=
\E_{P_0}[(Y-\eta)^2]+\E_{P_0}[(\eta-m)^2],
\]
which shows that \(m_0\) is the \(L^2(P_T)\)-projection of \(\eta\) onto the closed convex set \(\mathcal M_C\).

Now define the centered class
\[
\mathcal F_0:=\{g=m-m_0:\ m\in\mathcal M_C\}.
\]
For \(g\in\mathcal F_0\), let \(\ell_g(Z):=(Y-m_0(X)-g(X))^2-(Y-m_0(X))^2\), where \(Z=(X,Y)\). Since \(m_{n,\mathrm{iso}}^\star\) minimizes the empirical squared risk over \(\mathcal F_{\mathrm{iso}}\), on \(A_{n,C}\) the random element \(\hat g:=m_{n,\mathrm{iso}}^\star-m_0\) belongs to \(\mathcal F_0\) and satisfies
\[
\Pl\{\ell_{\hat g}\}\le 0=\Pl\{\ell_0\}.
\]
Moreover,
\[
\E_{P_0}[\ell_g]
=
\|g\|_{2,\Px}^2-2\E_{P_0}[(Y-m_0)g].
\]
Because \(g\) is measurable with respect to \(T\), the tower property gives
\[
\E_{P_0}[(Y-m_0)g]
=
\E_{P_0}[(\eta-m_0)g].
\]
Since \(m_0\) is the \(L^2(P_T)\)-projection of \(\eta\) onto the convex set \(\mathcal M_C\), the Hilbert-space projection inequality yields
\[
\E_{P_0}[(\eta-m_0)(m-m_0)]\le 0
\qquad\text{for every }m\in\mathcal M_C.
\]
Hence
\[
\E_{P_0}[\ell_g] \ge \|g\|_{2,\Px}^2,
\]
so squared loss has the required quadratic margin around \(m_0\).

We next bound the modulus of continuity of the empirical process indexed by the localized classes
\[
\mathcal F_\delta:=\{g\in\mathcal F_0:\|g\|_{2,\Px}\le \delta\}.
\]
The bracketing bound for bounded monotone functions gives
\[
\log N_{[]}(\varepsilon,\mathcal F_\delta,L_2(\Px))
\lesssim
\varepsilon^{-1},
\qquad 0<\varepsilon<\delta.
\]
Consequently,
\[
J_{[]}(\delta,\mathcal F_\delta,L_2(\Px))\lesssim \delta^{1/2};
\]
see, e.g., \citet{vdvwellner1996}.

Write \(\xi:=Y-m_0(X)\), so that \(\ell_g=g^2-2\xi g\). By Condition~\ref{cond::subgaus} and the bounds on \(\eta\), \(m_0\), and \(g\), there exists \(A_C<\infty\) such that
\begin{equation}
\label{eq:iso-multiplier-moments}
\E_{P_0}[|\xi g|^k]
\le k!A_C^k\|g\|_{2,\Px}^2,
\qquad g\in\mathcal F_0,\quad k\ge2.
\end{equation}
The same inequality holds for bounded bracket widths. Thus multiplication by \(\xi\) maps \(L^2(\Px)\) brackets to Bernstein-norm brackets with widths bounded by a constant multiple of their original widths; see Lemmas 8--9 of \citet{Bibaut2019}. The squared class \(\{g^2:g\in\mathcal F_\delta\}\) is a bounded Lipschitz image of \(\mathcal F_\delta\). Applying Lemmas 3.4.2--3.4.3 of \citet{vdvwellner1996} to these two classes gives
\begin{equation}
\label{eq:iso-local-modulus}
E^\star\sup_{g\in\mathcal F_\delta}|(\Pl-P_0)\ell_g|
\lesssim n^{-1/2}\delta^{1/2}+n^{-1}\delta^{-1}.
\end{equation}
For \(\delta\ge n^{-1/3}\), the right-hand side is bounded by a constant times \(n^{-1/2}\delta^{1/2}\).

Set \(\varepsilon_n:=n^{-1/3}\). The empirical risk inequality, the quadratic margin, and \eqref{eq:iso-local-modulus} imply, by peeling into shells \(2^s\varepsilon_n\le\|\hat g\|_{2,\Px}<2^{s+1}\varepsilon_n\) and applying Markov's inequality, that
\[
\begin{aligned}
\Pr\{\|\hat g\|_{2,\Px}\ge2^S\varepsilon_n,\ A_{n,C}\}
&\lesssim
\sum_{s\ge S}
\frac{n^{-1/2}(2^{s+1}\varepsilon_n)^{1/2}}
{2^{2s}\varepsilon_n^2}\\
&\lesssim\sum_{s\ge S}2^{-3s/2}
\longrightarrow0
\qquad(S\to\infty).
\end{aligned}
\]
Hence
\[
\|m_{n,\mathrm{iso}}^\star-m_0\|_{2,\Px}=O_p(n^{-1/3}),
\qquad
\|m_{n,\mathrm{iso}}^\star-m_0\|_{2,\Px}^2=O_p(n^{-2/3}).
\]
Since \(P_0(A_{n,C}^c)\le \zeta\) for all large \(n\), and \(\zeta>0\) was arbitrary, the same rate holds unconditionally, proving \cref{lem:isorate}.
\end{proof}

\begin{theorem}[Calibration error of isotonic post-processing]
\label{thm:iso-postcal}
Suppose \cref{assump:design,assump:isoreg,assump:eff-bv} hold. Let
\[
m_{0\mid n}(X):=\E_{P_0}[Y\mid m_{n,\mathrm{iso}}^\star(X)].
\]
Then
\[
\|m_{0\mid n}-m_{n,\mathrm{iso}}^\star\|_{2,\Px}^2=O_p(n^{-2/3}).
\]
\end{theorem}

\begin{proof}
The proof follows that of Theorem 1 in \cite{van2023causal}; see also \cite{van2025venn}. Let
\[
T:=\mhat(X),\qquad S_n:=m_{n,\mathrm{iso}}^\star(X),\qquad \eta_n(s):=\E_{P_0}[Y\mid S_n=s],
\]
and define
\[
\Delta_n(X):=m_{0\mid n}(X)-m_{n,\mathrm{iso}}^\star(X)=\eta_n(S_n)-S_n.
\]

Conditional on the labeled sample, \(h_n(s):=\eta_n(s)-s\) is a measurable function of \(s\). Applying \cref{prop:iso-repr} pathwise with \(h=h_n\) yields
\[
\Pl\{\Delta_n(X)\{Y-S_n\}\}=0.
\]
Also,
\[
\E_{P_0}[\Delta_n(X)\{Y-S_n\}]
=
\E_{P_0}[\Delta_n(X)\,\E_{P_0}[Y-S_n\mid S_n]]
=
\E_{P_0}[\Delta_n^2]
=
\|\Delta_n\|_{2,\Px}^2.
\]
Hence
\[
\|\Delta_n\|_{2,\Px}^2
=
(\E_{P_0}-\Pl)\{\Delta_n(X)\{Y-S_n\}\}.
\]

Fix \(\zeta>0\), and work on the event \(A_{n,C}\) from the proof convention. On this event, \(S_n\) is uniformly bounded by \(C\). By \cref{lem:bv-post}, the function \(\eta_n\) belongs to the bounded-variation class
\[
\mathcal F_{\mathrm{TV},C}
:=
\{f:\R\to\R:\|f\|_\infty\le C,\ \|f\|_{TV}\le V\},
\]
where \(V<\infty\) is the variation bound in \cref{assump:eff-bv}. Let
\[
\mathcal F_{\mathrm{iso},C}
:=
\{f:\R\to\R:\text{$f$ nondecreasing and }\|f\|_\infty\le C\}.
\]
Since \(S_n=f_n(T)\) for a nondecreasing fitted isotonic map \(f_n\), and composition with a monotone map preserves bounded variation, the function
\[
t\mapsto \eta_n\{f_n(t)\}
\]
also belongs to \(\mathcal F_{\mathrm{TV},C}\), after enlarging \(C\) if necessary. Thus \(\eta_n(S_n)=h_1(X)\) for some \(h_1=f_1\circ\mhat\) with \(f_1\in\mathcal F_{\mathrm{TV},C}\), while \(S_n=h_2(X)\) for \(h_2=f_2\circ\mhat\) with \(f_2=f_n\in\mathcal F_{\mathrm{iso},C}\).

Define the localized product class
\[
\mathcal G_{n,C,\delta}
:=
\Bigl\{
(h_1-h_2)(Y-h_2):
h_1=f_1\circ\mhat,\ h_2=f_2\circ\mhat,\
f_1\in\mathcal F_{\mathrm{TV},C},\
f_2\in\mathcal F_{\mathrm{iso},C},\
\|h_1-h_2\|_{2,\Px}\le \delta
\Bigr\}.
\]
Then
\[
\Delta_n(X)\{Y-S_n\}\in\mathcal G_{n,C,\delta_n},
\qquad
\delta_n:=\|\Delta_n\|_{2,\Px}.
\]

Write \(\eta(T):=\E_{P_0}[Y\mid T]\) and \(\varepsilon:=Y-\eta(T)\). Each element of \(\mathcal G_{n,C,\delta}\) decomposes as
\[
(h_1-h_2)(Y-h_2)
=
(h_1-h_2)\varepsilon+(h_1-h_2)\{\eta(T)-h_2\}.
\]
Bounded-variation functions are differences of bounded monotone functions, so their bracketing entropy is \(O(\epsilon^{-1})\). The second term is a bounded Lipschitz transformation of \((h_1,h_2)\). For the first term, Condition~\ref{cond::subgaus} gives the moment bound in \eqref{eq:iso-multiplier-moments}, with \(h_1-h_2\) in place of \(g\) and \(\varepsilon\) in place of \(\xi\). Its Bernstein-norm bracket widths are therefore controlled by \(L^2(\Px)\) bracket widths. The second term has \(L^2(P_0)\) norm \(O(\delta)\), and the first has Bernstein norm \(O(\delta)\); both bracketing entropy integrals are \(O(\delta^{1/2})\). Lemmas 3.4.2--3.4.3 of \citet{vdvwellner1996} give
\[
E^\star\sup_{g\in\mathcal G_{n,C,\delta}}|(\Pl-P_0)\{g\}|
\lesssim
n^{-1/2}\delta^{1/2}+n^{-1}\delta^{-1}.
\]
For \(\delta\ge n^{-1/3}\), this is bounded by a constant times \(n^{-1/2}\delta^{1/2}\).

Set \(\varepsilon_n:=n^{-1/3}\). For \(s\ge 0\), let
\[
B_s:=\{2^s\varepsilon_n\le \delta_n<2^{s+1}\varepsilon_n\}.
\]
On \(B_s\cap A_{n,C}\),
\[
2^{2s}\varepsilon_n^2
\le
\delta_n^2
\le
\sup_{g\in\mathcal G_{n,C,2^{s+1}\varepsilon_n}}|(\Pl-P_0)\{g\}|.
\]
Therefore, by Markov's inequality,
\[
P_0(B_s\cap A_{n,C})
\lesssim
\frac{E^\star\sup_{g\in\mathcal G_{n,C,2^{s+1}\varepsilon_n}}|(\Pl-P_0)\{g\}|}{2^{2s}\varepsilon_n^2}
\lesssim
\frac{n^{-1/2}(2^{s+1}\varepsilon_n)^{1/2}}{2^{2s}\varepsilon_n^2}
\lesssim
2^{-3s/2},
\]
since \(n^{-1/2}\varepsilon_n^{-3/2}=1\). Summing over \(s\ge S\) yields
\[
P_0(\delta_n\ge 2^S\varepsilon_n,\ A_{n,C})
\lesssim
\sum_{s=S}^\infty 2^{-3s/2}
\xrightarrow[S\to\infty]{} 0.
\]
Because \(P_0(A_{n,C}^c)\le \zeta\) for all large \(n\), and \(\zeta>0\) is arbitrary, it follows that
\[
\|\Delta_n\|_{2,\Px}=O_p(n^{-1/3}),
\qquad
\|\Delta_n\|_{2,\Px}^2=O_p(n^{-2/3}).
\]
This is exactly the claimed calibration-error rate.
\end{proof}

\begin{lemma}[Centered empirical-process bounds for isotonic calibration]
\label{lem:eprate}
Suppose \cref{assump:design,assump:isoreg} holds. Then
\[
(\Pl-\Px)\{m_{n,\mathrm{iso}}^\star(X)-m_0(X)\}=O_p(n^{-2/3}),
\qquad
(\Pu-\Px)\{m_{n,\mathrm{iso}}^\star(\widetilde X)-m_0(\widetilde X)\}=O_p\bigl(n^{-1/3}N^{-1/2}\bigr).
\]
\end{lemma}

\begin{proof}[Proof of \cref{lem:eprate}]
By \cref{lem:isorate},
\[
\|m_{n,\mathrm{iso}}^\star-m_0\|_{2,\Px}^2=O_p(n^{-2/3}).
\]
To control the centered empirical-process term, work on the event \(A_{n,C}\) from the proof convention. For fixed \(R<\infty\), consider the localized class
\[
\mathcal F_{n,C}
:=
\{g=f\circ\mhat-m_0:\ f\in\mathcal F_{\mathrm{iso}},\ \|f\circ\mhat\|_\infty\le C,\ \|g\|_{2,\Px}\le Rn^{-1/3}\}.
\]
By \cref{lem:isorate}, the probability of \(A_{n,C}\cap\{m_{n,\mathrm{iso}}^\star-m_0\notin\mathcal F_{n,C}\}\) can be made arbitrarily small by taking \(R\) large. Because \(\mhat\) is fixed, \(\mathcal F_{n,C}\) is a translated class of bounded monotone functions of a one-dimensional score. The monotone bracketing bound implies
\[
J_{[]}\!\left(r,\mathcal F_{n,C},L_2(P_0)\right)\lesssim r^{1/2}.
\]
Applying Lemma 3.4.2 of \citet{vdvwellner1996} at radius \(r_n=Rn^{-1/3}\) yields
\[
\sup_{g\in\mathcal F_{n,C}} |(\Pl-\Px)\{g\}|
=
O_p\!\left(n^{-1/2}r_n^{1/2}+n^{-1}r_n^{-1}\right)
=
O_p(n^{-2/3}).
\]
In particular,
\[
(\Pl-\Px)\{m_{n,\mathrm{iso}}^\star-m_0\}=O_p(n^{-2/3}).
\]
Since \(P_0(A_{n,C}^c)\le \zeta\) for all large \(n\), and \(\zeta>0\) is arbitrary, the same bound holds unconditionally.

For the unlabeled sample, conditional on the labeled data, \(m_{n,\mathrm{iso}}^\star-m_0\) is fixed and the unlabeled sample is independent, so
\[
(\Pu-\Px)\{m_{n,\mathrm{iso}}^\star-m_0\}
=
O_p\left(N^{-1/2}\|m_{n,\mathrm{iso}}^\star-m_0\|_{2,\Px}\right)
=
O_p\bigl(n^{-1/3}N^{-1/2}\bigr)
\]
by \cref{lem:isorate}.
\end{proof}

\begin{lemma}[Centered empirical-process bounds for isotonic post-processing]
\label{lem:post-eprate}
Suppose \cref{assump:design,assump:isoreg,assump:eff-bv} hold. Let
\[
\Delta_n(X):=m_{0\mid n}(X)-m_{n,\mathrm{iso}}^\star(X).
\]
Then
\[
\bigl(\Pl-\Px\bigr)\{\Delta_n(X)\}=O_p(n^{-2/3}),
\qquad
\bigl(\Pu-\Px\bigr)\{\Delta_n(\widetilde X)\}=O_p\bigl(n^{-1/3}N^{-1/2}\bigr).
\]
\end{lemma}

\begin{proof}
By \cref{thm:iso-postcal},
\[
\|\Delta_n\|_{2,\Px}^2=O_p(n^{-2/3}).
\]
Let \(V<\infty\) be the variation bound in \cref{assump:eff-bv}. Fix \(\zeta>0\), and work on the event \(A_{n,C}\). On this event, \(m_{n,\mathrm{iso}}^\star\) is uniformly bounded by \(C\). By \cref{lem:bv-post}, the function
\[
s\mapsto \eta_n(s):=\E_{P_0}[Y\mid m_{n,\mathrm{iso}}^\star(X)=s]
\]
has total variation at most \(V\) and is bounded by \(C\) on the support of \(m_{n,\mathrm{iso}}^\star(X)\). Since \(m_{n,\mathrm{iso}}^\star(X)=f_n(T)\) with \(T=\mhat(X)\) and \(f_n\) nondecreasing, it follows that
\[
m_{0\mid n}(X)=\eta_n\{f_n(T)\}
\]
is a bounded-variation transform of \(T\), while \(m_{n,\mathrm{iso}}^\star(X)\) is a bounded monotone transform of \(T\).

Consider the classes
\[
\mathcal F_{\mathrm{TV},C}
:=
\{f:\R\to\R:\|f\|_\infty\le C,\ \|f\|_{TV}\le V\},
\]
\[
\mathcal F_{\mathrm{iso},C}
:=
\{f:\R\to\R:\text{$f$ nondecreasing and }\|f\|_\infty\le C\}.
\]
For fixed \(R<\infty\), define
\[
\mathcal H_{n,C}
:=
\Bigl\{
h_1-h_2:
h_1=f_1\circ\mhat,\ h_2=f_2\circ\mhat,\
f_1\in\mathcal F_{\mathrm{TV},C},\
f_2\in\mathcal F_{\mathrm{iso},C},\
\|h_1-h_2\|_{2,\Px}\le Rn^{-1/3}
\Bigr\}.
\]
By \cref{thm:iso-postcal}, the probability of \(A_{n,C}\cap\{\Delta_n\notin\mathcal H_{n,C}\}\) can be made arbitrarily small by taking \(R\) large.

Because bounded-variation functions are differences of bounded monotone functions, the localized bracketing entropy of \(\mathcal H_{n,C}\) obeys the same bound as in the proof of \cref{thm:iso-postcal}, namely
\[
J_{[]}(\delta,\mathcal H_{n,C},L_2(P_0))\lesssim \delta^{1/2}.
\]
Applying Lemma 3.4.2 of \citet{vdvwellner1996} at radius \(r_n=Rn^{-1/3}\) gives
\[
\sup_{h\in\mathcal H_{n,C}}|(\Pl-\Px)h|
=O_p\!\left(n^{-1/2}r_n^{1/2}+n^{-1}r_n^{-1}\right)
=O_p(n^{-2/3}).
\]
Taking \(R\) large gives the same bound for \((\Pl-\Px)\Delta_n\) on \(A_{n,C}\).
Since \(P_0(A_{n,C}^c)\le \zeta\) for all large \(n\), and \(\zeta>0\) was arbitrary, the same bound holds unconditionally.

For the unlabeled sample, conditional on the labeled data, \(\Delta_n\) is fixed and the unlabeled sample is independent. Hence
\[
\bigl(\Pu-\Px\bigr)\{\Delta_n(\widetilde X)\}
=
O_p\left(N^{-1/2}\|\Delta_n\|_{2,\Px}\right)
=
O_p\bigl(n^{-1/3}N^{-1/2}\bigr)
\]
by \cref{thm:iso-postcal}.
\end{proof}

\section{Proof of asymptotic linearity}
\label{app:al}

\begin{proof}[Proof of \cref{thm:al}]

By \cref{lem:center}, \(\E_{\Px}[m_0]=\psi_0\). The isotonic fit has mean-zero labeled residuals by \cref{lem:block}. Applying \cref{lem:aipw-comparison} with \(g=m_0\) gives
\[
\widehat{\psi}_{\mathrm{iso}}
=
\rhon \Pl\left\{m_0+\rhon^{-1}(Y-m_0)\right\}
+
(1-\rhon)\Pu\{m_0(\widetilde X)\}
+
(1-\rhon)(\Pu-\Pl)\{m_{n,\mathrm{iso}}^\star-m_0\}.
\]
Subtracting \(\psi_0\) gives
\[
\begin{aligned}
\widehat{\psi}_{\mathrm{iso}}-\psi_0
&=
(\Pl-P_0)D_{m_0,n}^L
+
(\Pu-\Px)D_{m_0,n}^U\\
&\quad +
(1-\rhon)(\Pu-\Pl)\{m_{n,\mathrm{iso}}^\star-m_0\}.
\end{aligned}
\]
Also,
\[
(1-\rhon)(\Pu-\Pl)\{m_{n,\mathrm{iso}}^\star-m_0\}
=
(1-\rhon)(\Pu-\Px)\{m_{n,\mathrm{iso}}^\star-m_0\}
-
(1-\rhon)(\Pl-\Px)\{m_{n,\mathrm{iso}}^\star-m_0\}.
\]
By \cref{lem:eprate}, the labeled term is \(O_p(n^{-2/3})\) and the unlabeled term is \(O_p(n^{-1/3}N^{-1/2})\). Hence the remainder satisfies the bound in \cref{eq:al}. Moreover,
\[
\sqrt n(1-\rhon)n^{-2/3}\le n^{-1/6}
\]
and, writing \(t_n=N/n\),
\[
\sqrt n(1-\rhon)n^{-1/3}N^{-1/2}
=
n^{-1/3}\frac{\sqrt{t_n}}{1+t_n}
\le \frac12n^{-1/3}.
\]
Therefore \(R_{n,N}=o_p(n^{-1/2})\), which proves \cref{eq:al}.
\end{proof}

\begin{proof}[Proof of \cref{cor:normal}]
By \cref{thm:al},
\[
\sqrt n\,(\widehat{\psi}_{\mathrm{iso}}-\psi_0)
=
\frac{1}{\sqrt n}\sum_{i=1}^n D_{m_0,n}^L(X_i,Y_i)
+
\frac{\sqrt n}{N}\sum_{j=1}^N D_{m_0,n}^U(\widetilde X_j)
+
o_p(1).
\]
The labeled and unlabeled samples are independent and each summand has mean zero. Moreover,
\[
\frac{n}{N}\Var\{D_{m_0,n}^U(\widetilde X)\}
=
\rhon(1-\rhon)\Var\{m_0(X)\}.
\]
Since \(\rhon\to\rhozero\in[0,1)\), a triangular-array central limit theorem yields the stated asymptotic normal law with variance \(\sigma_0^2\).

For the variance estimator, let \(g_n:=m_{n,\mathrm{iso}}^\star-m_0\). By \cref{lem:isorate},
\[
\E_{\Px}[g_n^2]=\|g_n\|_{2,\Px}^2=O_p(n^{-2/3}).
\]
Applying Lemma 3.4.2 of \citet{vdvwellner1996} to the squared class gives
\[
\sup_{g\in\mathcal F_\delta}|(\Pl-\Px)\{g^2\}|
=O_p\!\left(n^{-1/2}\delta^{1/2}+n^{-1}\delta^{-1}\right),
\]
where \(\mathcal F_\delta\) is the localized class from the proof of \cref{lem:isorate}. By that lemma, on \(A_{n,C}\) the function \(g_n\) belongs to \(\mathcal F_{Rn^{-1/3}}\) with arbitrarily high probability for sufficiently large fixed \(R\). Taking \(\delta=Rn^{-1/3}\) gives
\[
(\Pl-\Px)\{g_n^2\}=O_p(n^{-2/3}).
\]
Hence
\[
\Pl(g_n^2)=O_p(n^{-2/3})=o_p(1).
\]

On the boundedness event \(A_{n,C}\), conditional on the labeled data, \(g_n\) is fixed and \(|g_n|\le 2C\), so
\[
(\Pu-\Px)\{g_n^2\}
=
O_p\left(N^{-1/2}\|g_n^2\|_{2,\Px}\right)
=
O_p\left(C\,N^{-1/2}\|g_n\|_{2,\Px}\right)
=
O_p\bigl(n^{-1/3}N^{-1/2}\bigr).
\]
Therefore
\[
\Pu(g_n^2)=O_p(n^{-2/3})+O_p\bigl(n^{-1/3}N^{-1/2}\bigr)=o_p(1).
\]
Since \(P_0(A_{n,C}^c)\le \zeta\) for all large \(n\) and arbitrary \(\zeta>0\), these bounds hold unconditionally. Also, \(\widehat{\psi}_{\mathrm{iso}}\to_p\psi_0\) follows from \cref{thm:al}.

Finally,
\[
\widehat D^L-D_{m_0,n}^L
=
-(1-\rhon)g_n-\rhon(\widehat{\psi}_{\mathrm{iso}}-\psi_0),
\]
and
\[
\widehat D^U-D_{m_0,n}^U
=
(1-\rhon)\{g_n(\widetilde X)-(\widehat{\psi}_{\mathrm{iso}}-\psi_0)\}.
\]
Because \(\Pl(g_n^2),\Pu(g_n^2)=o_p(1)\) and \(\widehat{\psi}_{\mathrm{iso}}\to_p\psi_0\), it follows that
\[
\Pl\!\left[\{\widehat D^L-D_{m_0,n}^L\}^2\right]=o_p(1),
\qquad
\frac{n}{N}\Pu\!\left[\{\widehat D^U-D_{m_0,n}^U\}^2\right]=o_p(1),
\]
where the second conclusion uses
\[
\frac{n}{N}(1-\rhon)^2=\rhon(1-\rhon)\le \frac14.
\]
The law of large numbers and \(\rhon\to\rhozero\) also give
\[
\Pl\{(D_{m_0,n}^L)^2\}
+
\frac{n}{N}\Pu\{(D_{m_0,n}^U)^2\}
\ \to_p\ \sigma_0^2.
\]
The empirical means of both estimated influence components converge to zero,
so within-sample centering does not change this limit.
Therefore \(\widehat\sigma^2\to_p \sigma_0^2\). The Wald interval claim then follows by Slutsky's theorem.
\end{proof}

\section{Proof of post-processing equivalence and reduced-model efficiency}
\label{app:eff}

\begin{proof}[Proof of Theorem~\ref{prop:eif-red}]
Set \(\Delta_n=m_{0\mid n}-m_{n,\mathrm{iso}}^\star\).
Lemma~\ref{lem:aipw-comparison} with \(g=m_{0\mid n}\) gives
\[
\widehat\psi_{\mathrm{eff}}-\widehat\psi_{\mathrm{iso}}
=(1-\rhon)(\Pu-\Pl)\Delta_n.
\]
By Lemma~\ref{lem:post-eprate}, the right-hand side is
\[
O_p((1-\rhon)n^{-2/3})
+O_p((1-\rhon)n^{-1/3}N^{-1/2})
=o_p(n^{-1/2}),
\]
using the same bounds as in the proof of Theorem~\ref{thm:al}.
The calibration-error rate is Theorem~\ref{thm:iso-postcal}.
Together these prove both conclusions.
\end{proof}

\begin{proof}[Proof of \cref{cor:efficiency}]
For part (i), if \(m_0=\muzero\) almost surely, then the asymptotic influence pair in \cref{thm:al} is
\[
D_{\mu_0,n}^L(X,Y)
=
Y-\psi_0-(1-\rhon)\{\mu_0(X)-\psi_0\},
\qquad
D_{\mu_0,n}^U(\widetilde X)
=
(1-\rhon)\{\mu_0(\widetilde X)-\psi_0\}.
\]
By \cref{prop:aipw-class}, this is the efficient influence pair in the unrestricted two-sample model. Therefore \(\widehat{\psi}_{\mathrm{iso}}\) is efficient for \(\psi_0\) in the full two-sample experiment.

For part (ii), fix a square-integrable nondecreasing transformation
\(g=\theta\circ\mhat\), and let \(g^\circ=g-\E_{\Px}g+\psi_0\).
Adding a constant preserves monotonicity, so population isotonic optimality
and Lemma~\ref{lem:center} give
\[
\Var(Y-m_0)=\E[(Y-m_0)^2]
\le\E[(Y-g^\circ)^2]=\Var(Y-g).
\]
Proposition~\ref{prop:aipw-class} now gives the efficiency comparison.

For part (iii), let
\[
S:=\mhat(X),
\qquad
\eta_0(s):=\E_{P_0}[Y\mid S=s].
\]
\cref{prop:aipw-class} applied to the reduced two-sample experiment with scalar covariate \(S\) shows that the efficient triangular influence pair in that reduced experiment is
\[
D_{\eta_0,n}^L(S,Y)
=
Y-\psi_0-(1-\rhon)\{\eta_0(S)-\psi_0\},
\qquad
D_{\eta_0,n}^U(\widetilde S)
=
(1-\rhon)\{\eta_0(\widetilde S)-\psi_0\}.
\]
If \(\eta_0\) is monotone increasing, then \(\eta_0\) belongs to the isotonic class and
\[
\E_{P_0}\bigl[\{Y-g(S)\}^2\bigr]
=
\E_{P_0}\bigl[\{Y-\eta_0(S)\}^2\bigr]
+
\E_{P_0}\bigl[\{\eta_0(S)-g(S)\}^2\bigr]
\]
for every measurable \(g\). Hence the population isotonic regression of \(Y\) on the score \(S=\mhat(X)\) is exactly \(\eta_0(S)\), so
\[
m_0(X)=\eta_0\{\mhat(X)\}
\qquad\text{almost surely.}
\]
Therefore, the influence pair in \cref{thm:al} coincides with the reduced-model efficient influence pair above, so \(\widehat{\psi}_{\mathrm{iso}}\) is efficient for \(\psi_0\) in the reduced experiment indexed by \(\mhat(X)\).
The isotonic class is closed under constant shifts and positive rescaling.
The corresponding population first-order conditions give
\(\E[Y-m_0]=\E[m_0(Y-m_0)]=0\), so
\(\Cov(Y,m_0)=\Var(m_0)\).
Substitution into \(\sigma_0^2\) gives the stated variance reduction relative
to the labeled mean.
\end{proof}

\section{Pooled i.i.d.\ Missing-Data Formulation}
\label{app:iid}

When \(\rhozero>0\), the joint two-sample setup can be embedded into a pooled i.i.d.\ model by introducing a sample-membership indicator \(R \in \{0,1\}\), where \(R=1\) denotes labeled observations and \(R=0\) denotes unlabeled observations. In that formulation one observes i.i.d.\ data \(O=(X,R,RY)\) with \(\Prob(R=1)=\rhozero\), and the calibrated estimator becomes
\[
\psihat = \mathbb P_M(\mtilde) + \mathbb P_M\left[\frac{R}{\rhozero}\{Y-\mtilde\}\right]
\]
up to the same calibration identity as in the main text. The efficient influence function in the reduced model observing only the score is
\[
\E[Y\mid m(X)]-\psi_0
+\frac{R}{\rhozero}\{Y-\E[Y\mid m(X)]\},
\]
which is the single-sample analogue of the efficient influence pair obtained by applying \cref{prop:aipw-class} to the reduced model with scalar covariate \(\mhat(X)\). Thus the usual i.i.d.\ missing-data formulation is a convenient special case, but the main text keeps the two-dataset structure explicit.

 \section{Additional empirical details}
\label{app:empirical}

\subsection{Small-sample comparisons and RePPI}
\label{app:small-n}

The small-sample experiment uses \(S\sim N(0,1)\),
\(p(S)=\{1+\exp(-5S)\}^{-1}\), and \(Y\mid S\sim\mathrm{Bernoulli}\{p(S)\}\).
Let \(C(t)=\min\{0.99,\max(0.01,t)\}\), \(\sigma(t)=(1+e^{-t})^{-1}\),
and \(Z\sim N(0,1)\) independently of \((S,Y)\).
The mildly miscalibrated score is
\(C[\sigma\{1.5\,\mathrm{logit}(C[p(S)+0.005Z])\}]\);
the strongly miscalibrated score is
\(C[-0.15+0.75\,\sigma\{3\,\mathrm{logit}(p(S))-0.5\}]\).
The simulation seed is 20260411. Figure~\ref{fig:small-n-diagnostics} reports
signed bias, coverage of nominal 95\% intervals, interval width, and relative
efficiency versus PPI++.

We implement two scalar-score versions of RePPI Algorithm~1
\citep{ji2025predictions}, using three-way sample splitting and
weights proportional to the evaluation-fold sizes. RePPI-RF uses a 100-tree random forest.
RePPI-MLP-CV selects hidden-layer widths from \((8)\), \((32,32)\), and
\((32,32,32)\), and penalties from \(\{10^{-2},10^{-4}\}\), by two-fold
validation MSE within each outer training fold. Both estimate the scalar
power-tuning coefficient on the evaluation fold, use the global factor
\(N/(n+N)\), and set the coefficient to zero when fitted-score variance is
numerically singular. The neural implementation uses the feature vector
\([1,m(X)]\). Its prediction map and the scalar single-fold calculation were
checked against the public implementation; the three-way aggregation follows
the paper's Algorithm~1.

RePPI-RF has lower RMSE than PPI++ in one of 14 settings; RePPI-MLP-CV does
so in five. For \(n\ge50\), the neural estimator's RMSE is
\(0.576\)--\(0.668\) times that of the labeled mean. At \(n=20\), the ratios
are \(1.450\) and \(9.690\) under mild and strong miscalibration.
Neither RePPI implementation has lower RMSE than smooth monotone calibration
in these settings. The largest neural-estimator error occurs when a
fitted-score variance of \(1.04\times10^{-7}\) produces a coefficient of
\(1238.2\). All replications are retained. Singular-variance fallbacks occur
in 175 random-forest replications and 46 neural-network replications;
the neural fits produce 2,120 convergence warnings.

\begin{figure}[p]
\centering
\includegraphics[width=\textwidth]{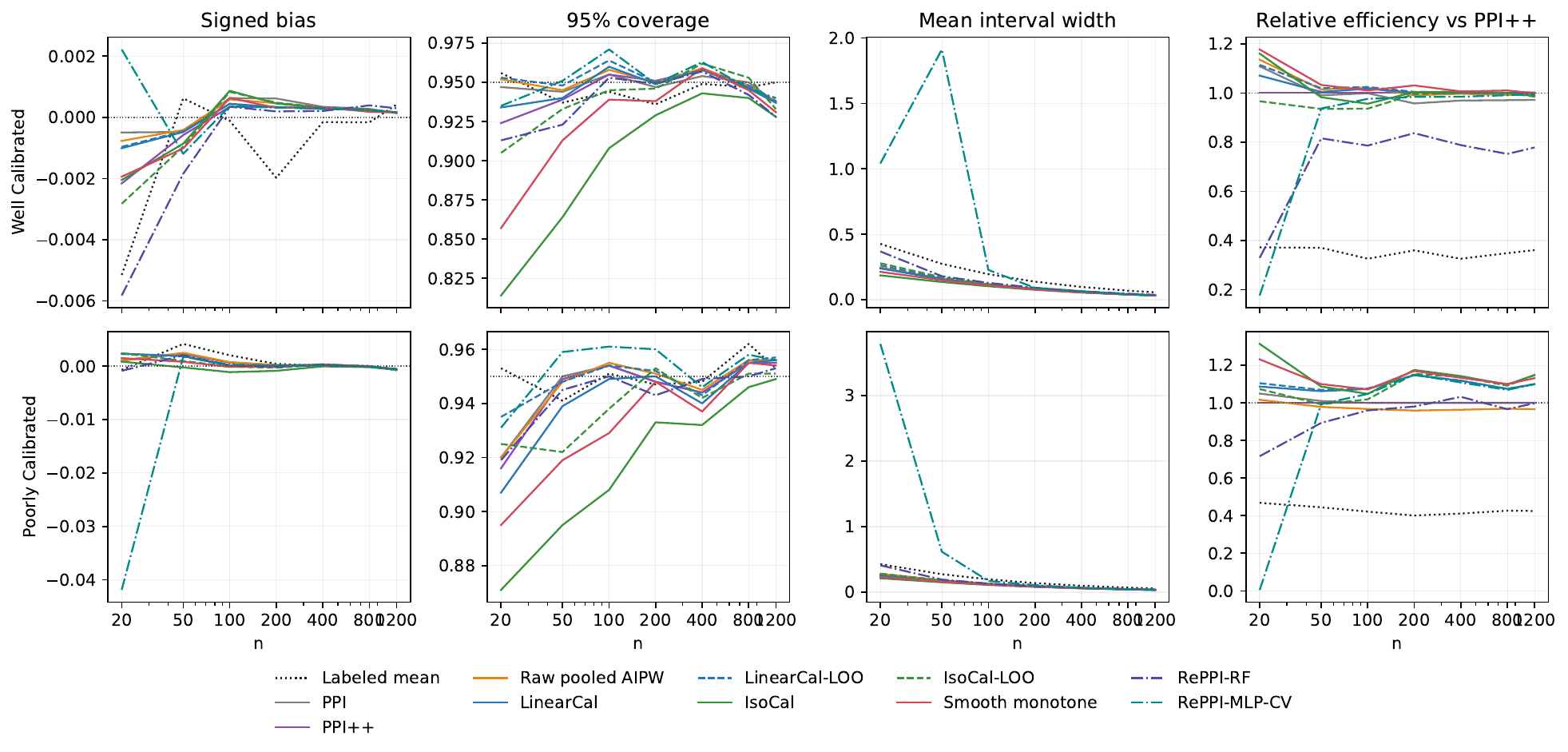}
\caption{Signed bias, coverage of nominal 95\% intervals, interval width,
and relative efficiency versus PPI++ in the small-sample experiment.
The numerical summary includes Monte Carlo standard errors.}
\label{fig:small-n-diagnostics}
\end{figure}

\begin{figure}[p]
\centering
\includegraphics[width=\textwidth]{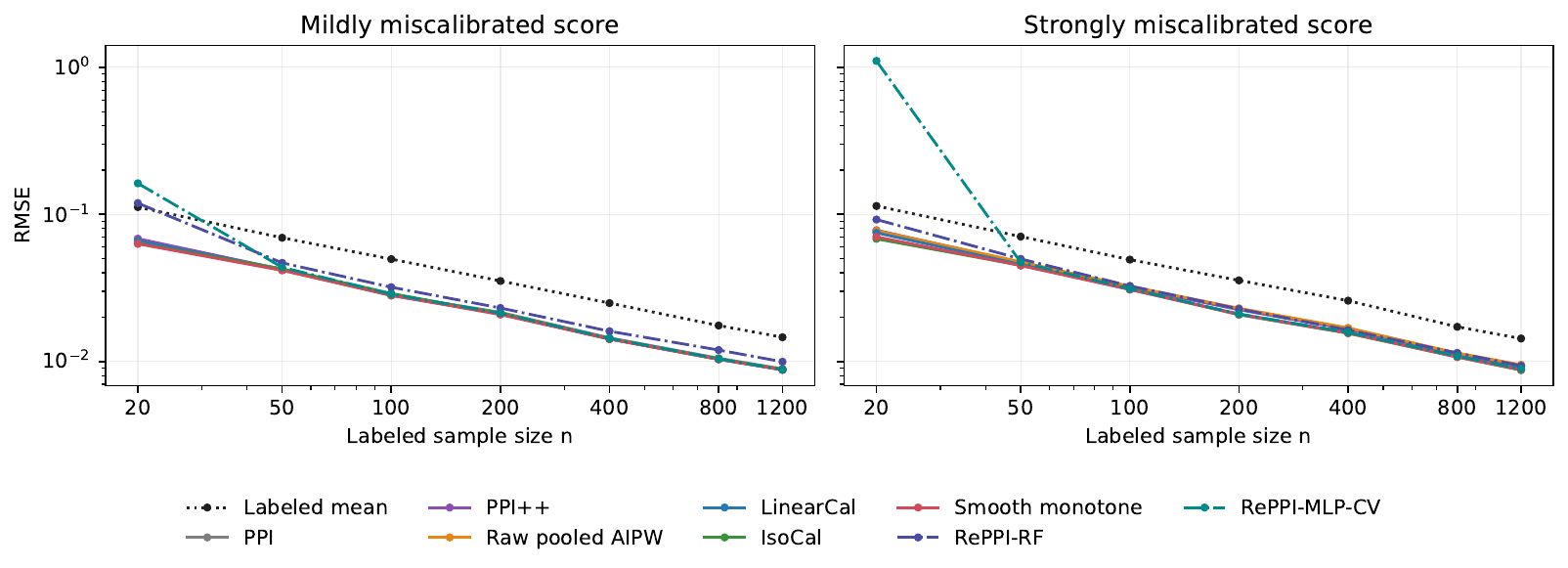}
\caption{Small-sample RMSE comparison including the random-forest and
cross-validated neural-network RePPI implementations.}
\label{fig:small-n-reppi}
\end{figure}

This appendix describes the benchmark construction used in \cref{sec:empirical}. We use the official \texttt{ppi\_py} real-data mean-estimation examples and preserve their labeled-sample-size grids. The \texttt{forest} experiment uses \(n \in \{50,100,\dots,500\}\). The \texttt{galaxies} and semisupervised \texttt{census\_income} benchmarks use \(n \in \{50,155,261,366,472,577,683,788,894,1000\}\). For each \(n\), the remaining observations form the unlabeled sample. We generate 500 random splits for each \(n\), evaluate all estimators on the same split, and use the full-sample mean as the benchmark target.

We compare the labeled-only mean, PPI, AIPW, PPI++, \textsc{AIPW-EM},
\texttt{LinearCal}, \texttt{MonoSpline}, \texttt{IsoCal}, and \texttt{AutoCal}.
The binary-outcome benchmarks also include imputation, Platt scaling,
and Venn--Abers shrinkage. PPI and PPI++ use the official \texttt{ppi\_py}
implementations. PPI++ restricts its coefficient to \([0,1]\), whereas
\textsc{AIPW-EM} minimizes the empirical variance without that constraint.

Calibration maps are fitted on the labeled pairs. We use the calibrated
scores in the AIPW estimator, retaining the residual correction after clipping
or binning. \texttt{LinearCal} fits an affine map; \texttt{MonoSpline} fits a
smooth nondecreasing spline with a roughness penalty. Both clip predictions
to the range of the labeled outcomes. \texttt{IsoCal} uses a minimum bin size
of 10. Platt scaling fits logistic regression to a stabilized logit transform
of the score. \texttt{AutoCal} selects among AIPW, linear calibration,
monotone splines, and isotonic calibration using 20-fold cross-validation
of the estimated influence-function variance. Each fold uses an unlabeled
subsample of size \(\min(N,10n)\) to evaluate this criterion.
We report Monte Carlo bias, variance, MSE, Wald-interval coverage,
and relative efficiency versus PPI.

\begin{table}[hbt]
\centering
\small
\resizebox*{!}{0.78\textheight}{\begin{tabular}{lllrccccc}
\toprule
Dataset & $n$ & $N$ & Estimator & Bias & Variance & MSE & Coverage & RelEff \\
\midrule
census\_income\_semisupervised\_mean & 50 & 380041 & AIPW & 356.4905 & 69836665.4410 & 69963750.9366 & 0.842 & 1.000 \\
census\_income\_semisupervised\_mean & 50 & 380041 & AIPW-EM & 260.9833 & 67728426.6042 & 67796538.8913 & 0.822 & 1.031 \\
census\_income\_semisupervised\_mean & 50 & 380041 & AutoCal & -123.2501 & 69242944.1124 & 69258134.7078 & 0.832 & 1.009 \\
census\_income\_semisupervised\_mean & 50 & 380041 & Labeled-only & 685.3532 & 105558402.2427 & 106028111.3137 & 0.860 & 0.662 \\
census\_income\_semisupervised\_mean & 50 & 380041 & IsoCal & -215.3832 & 69697400.2434 & 69743790.1536 & 0.820 & 1.002 \\
census\_income\_semisupervised\_mean & 50 & 380041 & LinearCal & 260.9294 & 67622601.4006 & 67690685.5430 & 0.816 & 1.033 \\
census\_income\_semisupervised\_mean & 50 & 380041 & MonoSpline & -482.0477 & 69082916.1270 & 69315286.1508 & 0.778 & 1.011 \\
census\_income\_semisupervised\_mean & 50 & 380041 & PPI & 356.4473 & 69834865.9486 & 69961920.5977 & 0.842 & 1.000 \\
census\_income\_semisupervised\_mean & 50 & 380041 & PPI++ & 150.8664 & 72953053.8328 & 72975814.4980 & 0.830 & 0.957 \\
census\_income\_semisupervised\_mean & 577 & 379514 & AIPW & 199.8136 & 6353884.5833 & 6393810.0648 & 0.916 & 1.000 \\
census\_income\_semisupervised\_mean & 577 & 379514 & AIPW-EM & 177.6067 & 6114918.2974 & 6146462.4536 & 0.912 & 1.039 \\
census\_income\_semisupervised\_mean & 577 & 379514 & AutoCal & 134.3746 & 6201029.6459 & 6219086.1828 & 0.910 & 1.024 \\
census\_income\_semisupervised\_mean & 577 & 379514 & Labeled-only & 213.2465 & 9407595.6062 & 9453069.6946 & 0.902 & 0.675 \\
census\_income\_semisupervised\_mean & 577 & 379514 & IsoCal & 155.0503 & 6162964.7883 & 6187005.3977 & 0.902 & 1.031 \\
census\_income\_semisupervised\_mean & 577 & 379514 & LinearCal & 179.8005 & 6116832.9552 & 6149161.1594 & 0.910 & 1.038 \\
census\_income\_semisupervised\_mean & 577 & 379514 & MonoSpline & 162.1952 & 6146069.6864 & 6172376.9541 & 0.904 & 1.033 \\
census\_income\_semisupervised\_mean & 577 & 379514 & PPI & 199.7932 & 6351836.6166 & 6391753.9369 & 0.916 & 1.000 \\
census\_income\_semisupervised\_mean & 577 & 379514 & PPI++ & 198.7409 & 6362047.6345 & 6401545.5819 & 0.916 & 0.998 \\
census\_income\_semisupervised\_mean & 1000 & 379091 & AIPW & 92.1091 & 3462051.4638 & 3470535.5425 & 0.916 & 0.999 \\
census\_income\_semisupervised\_mean & 1000 & 379091 & AIPW-EM & 61.4705 & 3288235.4312 & 3292014.0588 & 0.912 & 1.052 \\
census\_income\_semisupervised\_mean & 1000 & 379091 & AutoCal & 25.8207 & 3289467.9320 & 3290134.6381 & 0.922 & 1.052 \\
census\_income\_semisupervised\_mean & 1000 & 379091 & Labeled-only & 165.4763 & 5293296.2822 & 5320678.6872 & 0.914 & 0.654 \\
census\_income\_semisupervised\_mean & 1000 & 379091 & IsoCal & 51.2697 & 3271622.4413 & 3274251.0223 & 0.904 & 1.058 \\
census\_income\_semisupervised\_mean & 1000 & 379091 & LinearCal & 60.7501 & 3287796.5296 & 3291487.1058 & 0.912 & 1.052 \\
census\_income\_semisupervised\_mean & 1000 & 379091 & MonoSpline & 60.9626 & 3306575.6859 & 3310292.1289 & 0.910 & 1.046 \\
census\_income\_semisupervised\_mean & 1000 & 379091 & PPI & 91.9155 & 3460071.7259 & 3468520.1894 & 0.916 & 1.000 \\
census\_income\_semisupervised\_mean & 1000 & 379091 & PPI++ & 91.9155 & 3460071.7259 & 3468520.1894 & 0.916 & 1.000 \\
forest\_mean & 50 & 1546 & AIPW & -0.0012 & 0.0018 & 0.0018 & 0.928 & 1.006 \\
forest\_mean & 50 & 1546 & AIPW-EM & -0.0018 & 0.0019 & 0.0019 & 0.918 & 0.969 \\
forest\_mean & 50 & 1546 & AutoCal & -0.0028 & 0.0019 & 0.0019 & 0.924 & 0.965 \\
forest\_mean & 50 & 1546 & Labeled-only & 0.0017 & 0.0024 & 0.0024 & 0.954 & 0.762 \\
forest\_mean & 50 & 1546 & Imputation & -0.0733 & 0.0000 & 0.0054 & 0.000 & -- \\
forest\_mean & 50 & 1546 & IsoCal & -0.0030 & 0.0020 & 0.0020 & 0.930 & 0.934 \\
forest\_mean & 50 & 1546 & LinearCal & -0.0016 & 0.0019 & 0.0019 & 0.918 & 0.986 \\
forest\_mean & 50 & 1546 & MonoSpline & -0.0021 & 0.0019 & 0.0019 & 0.904 & 0.949 \\
forest\_mean & 50 & 1546 & Platt & -0.0003 & 0.0019 & 0.0019 & 0.900 & 0.948 \\
forest\_mean & 50 & 1546 & PPI & -0.0013 & 0.0018 & 0.0018 & 0.934 & 1.000 \\
forest\_mean & 50 & 1546 & PPI++ & -0.0064 & 0.0019 & 0.0019 & 0.928 & 0.969 \\
forest\_mean & 50 & 1546 & Venn-Abers & -0.0027 & 0.0020 & 0.0020 & 0.888 & 0.923 \\
forest\_mean & 300 & 1296 & AIPW & 0.0011 & 0.0002 & 0.0002 & 0.984 & 1.036 \\
forest\_mean & 300 & 1296 & AIPW-EM & 0.0007 & 0.0002 & 0.0002 & 0.982 & 1.012 \\
forest\_mean & 300 & 1296 & AutoCal & 0.0006 & 0.0002 & 0.0002 & 0.978 & 1.002 \\
forest\_mean & 300 & 1296 & Labeled-only & 0.0007 & 0.0003 & 0.0004 & 0.968 & 0.706 \\
forest\_mean & 300 & 1296 & Imputation & -0.0733 & 0.0000 & 0.0054 & 0.000 & -- \\
forest\_mean & 300 & 1296 & IsoCal & 0.0010 & 0.0002 & 0.0002 & 0.978 & 1.000 \\
forest\_mean & 300 & 1296 & LinearCal & 0.0010 & 0.0002 & 0.0002 & 0.984 & 1.018 \\
forest\_mean & 300 & 1296 & MonoSpline & 0.0010 & 0.0002 & 0.0002 & 0.974 & 1.004 \\
forest\_mean & 300 & 1296 & Platt & 0.0010 & 0.0002 & 0.0002 & 0.982 & 1.020 \\
forest\_mean & 300 & 1296 & PPI & 0.0011 & 0.0002 & 0.0002 & 0.986 & 1.000 \\
forest\_mean & 300 & 1296 & PPI++ & -0.0002 & 0.0002 & 0.0002 & 0.980 & 0.991 \\
forest\_mean & 300 & 1296 & Venn-Abers & 0.0011 & 0.0002 & 0.0002 & 0.978 & 1.015 \\
forest\_mean & 500 & 1096 & AIPW & 0.0001 & 0.0001 & 0.0001 & 0.980 & 1.172 \\
forest\_mean & 500 & 1096 & AIPW-EM & -0.0002 & 0.0001 & 0.0001 & 0.980 & 1.188 \\
forest\_mean & 500 & 1096 & AutoCal & -0.0003 & 0.0001 & 0.0001 & 0.986 & 1.180 \\
forest\_mean & 500 & 1096 & Labeled-only & 0.0001 & 0.0002 & 0.0002 & 0.972 & 0.879 \\
forest\_mean & 500 & 1096 & Imputation & -0.0733 & 0.0000 & 0.0054 & 0.000 & -- \\
forest\_mean & 500 & 1096 & IsoCal & -0.0000 & 0.0001 & 0.0001 & 0.986 & 1.215 \\
forest\_mean & 500 & 1096 & LinearCal & 0.0001 & 0.0001 & 0.0001 & 0.982 & 1.191 \\
forest\_mean & 500 & 1096 & MonoSpline & 0.0000 & 0.0001 & 0.0001 & 0.984 & 1.207 \\
forest\_mean & 500 & 1096 & Platt & 0.0001 & 0.0001 & 0.0001 & 0.984 & 1.203 \\
forest\_mean & 500 & 1096 & PPI & 0.0000 & 0.0002 & 0.0002 & 0.982 & 1.000 \\
forest\_mean & 500 & 1096 & PPI++ & -0.0006 & 0.0001 & 0.0001 & 0.978 & 1.179 \\
forest\_mean & 500 & 1096 & Venn-Abers & -0.0001 & 0.0001 & 0.0001 & 0.978 & 1.206 \\
galaxies\_mean & 50 & 16693 & AIPW & 0.0035 & 0.0018 & 0.0018 & 0.882 & 1.000 \\
galaxies\_mean & 50 & 16693 & AIPW-EM & 0.0035 & 0.0019 & 0.0020 & 0.864 & 0.944 \\
galaxies\_mean & 50 & 16693 & AutoCal & 0.0019 & 0.0018 & 0.0018 & 0.896 & 0.997 \\
galaxies\_mean & 50 & 16693 & Labeled-only & 0.0014 & 0.0041 & 0.0041 & 0.890 & 0.451 \\
galaxies\_mean & 50 & 16693 & Imputation & -0.0433 & 0.0000 & 0.0019 & 0.000 & -- \\
galaxies\_mean & 50 & 16693 & IsoCal & -0.0011 & 0.0020 & 0.0020 & 0.850 & 0.908 \\
galaxies\_mean & 50 & 16693 & LinearCal & 0.0030 & 0.0019 & 0.0019 & 0.862 & 0.973 \\
galaxies\_mean & 50 & 16693 & MonoSpline & 0.0026 & 0.0020 & 0.0020 & 0.830 & 0.917 \\
galaxies\_mean & 50 & 16693 & Platt & 0.0033 & 0.0019 & 0.0019 & 0.856 & 0.966 \\
galaxies\_mean & 50 & 16693 & PPI & 0.0035 & 0.0018 & 0.0018 & 0.882 & 1.000 \\
galaxies\_mean & 50 & 16693 & PPI++ & -0.0020 & 0.0020 & 0.0020 & 0.864 & 0.914 \\
galaxies\_mean & 50 & 16693 & Venn-Abers & 0.0022 & 0.0020 & 0.0020 & 0.830 & 0.907 \\
galaxies\_mean & 577 & 16166 & AIPW & -0.0006 & 0.0001 & 0.0001 & 0.906 & 1.004 \\
galaxies\_mean & 577 & 16166 & AIPW-EM & -0.0007 & 0.0001 & 0.0001 & 0.906 & 1.005 \\
galaxies\_mean & 577 & 16166 & AutoCal & -0.0008 & 0.0001 & 0.0001 & 0.918 & 0.989 \\
galaxies\_mean & 577 & 16166 & Labeled-only & -0.0011 & 0.0003 & 0.0003 & 0.910 & 0.436 \\
galaxies\_mean & 577 & 16166 & Imputation & -0.0433 & 0.0000 & 0.0019 & 0.000 & -- \\
galaxies\_mean & 577 & 16166 & IsoCal & -0.0007 & 0.0001 & 0.0001 & 0.894 & 0.986 \\
galaxies\_mean & 577 & 16166 & LinearCal & -0.0007 & 0.0001 & 0.0001 & 0.908 & 1.005 \\
galaxies\_mean & 577 & 16166 & MonoSpline & -0.0006 & 0.0001 & 0.0001 & 0.892 & 0.988 \\
galaxies\_mean & 577 & 16166 & Platt & -0.0007 & 0.0001 & 0.0001 & 0.896 & 1.002 \\
galaxies\_mean & 577 & 16166 & PPI & -0.0006 & 0.0001 & 0.0001 & 0.900 & 1.000 \\
galaxies\_mean & 577 & 16166 & PPI++ & -0.0012 & 0.0001 & 0.0001 & 0.902 & 0.988 \\
galaxies\_mean & 577 & 16166 & Venn-Abers & -0.0006 & 0.0001 & 0.0001 & 0.892 & 0.986 \\
galaxies\_mean & 1000 & 15743 & AIPW & 0.0002 & 0.0001 & 0.0001 & 0.948 & 0.987 \\
galaxies\_mean & 1000 & 15743 & AIPW-EM & 0.0002 & 0.0001 & 0.0001 & 0.948 & 0.980 \\
galaxies\_mean & 1000 & 15743 & AutoCal & 0.0001 & 0.0001 & 0.0001 & 0.944 & 0.977 \\
galaxies\_mean & 1000 & 15743 & Labeled-only & 0.0003 & 0.0002 & 0.0002 & 0.884 & 0.357 \\
galaxies\_mean & 1000 & 15743 & Imputation & -0.0433 & 0.0000 & 0.0019 & 0.000 & -- \\
galaxies\_mean & 1000 & 15743 & IsoCal & 0.0001 & 0.0001 & 0.0001 & 0.936 & 0.951 \\
galaxies\_mean & 1000 & 15743 & LinearCal & 0.0002 & 0.0001 & 0.0001 & 0.946 & 0.980 \\
galaxies\_mean & 1000 & 15743 & MonoSpline & 0.0002 & 0.0001 & 0.0001 & 0.944 & 0.969 \\
galaxies\_mean & 1000 & 15743 & Platt & 0.0002 & 0.0001 & 0.0001 & 0.942 & 0.979 \\
galaxies\_mean & 1000 & 15743 & PPI & 0.0002 & 0.0001 & 0.0001 & 0.948 & 1.000 \\
galaxies\_mean & 1000 & 15743 & PPI++ & -0.0002 & 0.0001 & 0.0001 & 0.946 & 0.977 \\
galaxies\_mean & 1000 & 15743 & Venn-Abers & 0.0002 & 0.0001 & 0.0001 & 0.942 & 0.952 \\
\bottomrule
\end{tabular}
}
\caption{Representative numerical summary at the smallest, middle, and largest labeled sample sizes for each reproduced benchmark. The reported quantities are Monte Carlo bias, empirical variance of the point estimator, MSE, Wald-interval coverage, and relative efficiency versus PPI for the displayed benchmark set, including PPI++, \textsc{AIPW-EM}, \texttt{AutoCal}, \texttt{MonoSpline}, and \texttt{IsoCal}. Relative efficiency is not reported when the estimator has zero variance.}
\label{tab:ppi-benchmarks-full}
\end{table}

\begin{figure}[hbt]
\centering
\includegraphics[width=\textwidth]{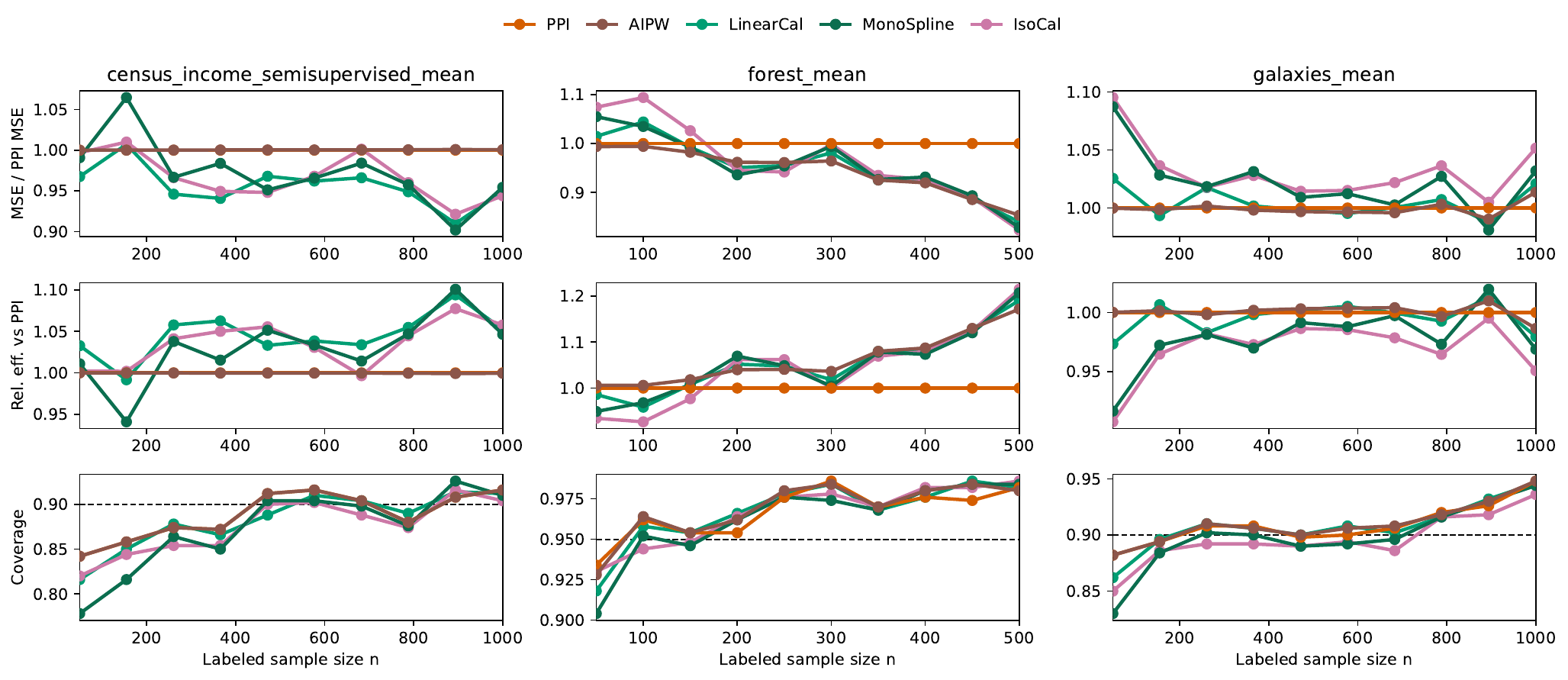}
\caption{Calibration-focused appendix comparison across the reproduced PPI benchmarks, restricted to PPI, AIPW, \texttt{LinearCal}, \texttt{MonoSpline}, and \texttt{IsoCal}. The panels report normalized MSE relative to PPI, relative efficiency versus PPI, and coverage across labeled-sample-size regimes.}
\label{fig:ppi-calibration-appendix}
\end{figure}

\subsection{LLM score ordering and held-out calibration}
\label{app:llm-ordering}

For MATH responses evaluated by ArmoRM, empirical correctness increases across
twelve score bins, from \(0.078\) in the lowest bin to \(0.920\) in the
highest. Platt calibration fitted on 100 randomly selected labels reduces
mean squared prediction error from \(0.4683\) to \(0.2020\) on the remaining
held-out observations, a reduction of \(56.9\%\).

\begin{figure}[htb]
\centering
\includegraphics[width=\textwidth]{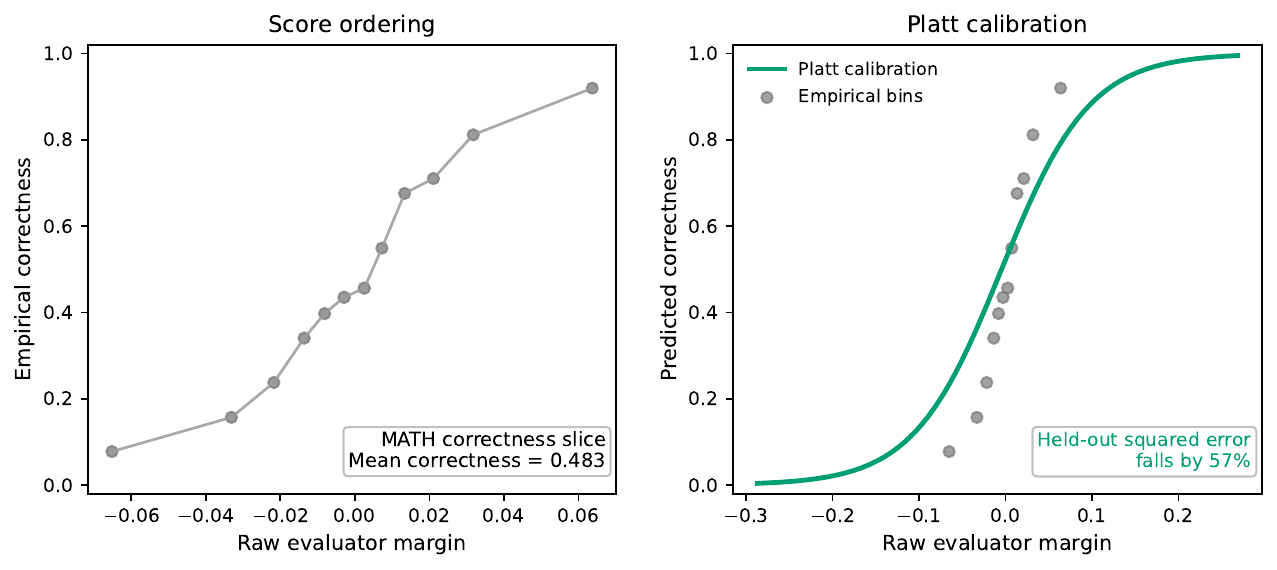}
\caption{ArmoRM scores and observed MATH correctness. Left: empirical
correctness in twelve score bins. Right: Platt calibration fitted on 100
labeled observations; prediction error is evaluated on the remaining
observations.}
\label{fig:llm-ordering}
\end{figure}

\subsection{LLM evaluator breakdown}
\label{app:llm-eval-by-evaluator}

The largest raw-score MSE inflation occurs for Skywork and Athene. With ArmoRM, the main-text estimators have similar MSEs. The evaluator-specific results therefore show where calibration provides the largest gains.

\IfFileExists{assets/table_llm_summary_appendix.tex}{
\begin{table}[hbt]
\centering
\small
\resizebox{\textwidth}{!}{\begin{tabular}{lllccc}
\toprule
Track & $n$ & Estimator & MSE / PPI & Label savings & Coverage \\
\midrule
PPE Correctness & 100 & Labeled-only & 0.346 & 0.000 & 0.927 \\
PPE Correctness & 100 & PPI & 1.000 & 0.010 & 0.904 \\
PPE Correctness & 100 & AIPW & 0.774 & 0.009 & 0.906 \\
PPE Correctness & 100 & PPI++ & 0.330 & 0.070 & 0.929 \\
PPE Correctness & 100 & AIPW-EM & 0.309 & 0.095 & 0.930 \\
PPE Correctness & 100 & LinearCal & 0.305 & 0.101 & 0.931 \\
PPE Correctness & 100 & AutoCal & 0.310 & 0.094 & 0.929 \\
PPE Correctness & 100 & MonoSpline & 0.309 & 0.092 & 0.926 \\
PPE Correctness & 100 & IsoCal & 0.309 & 0.089 & 0.925 \\
PPE Correctness & 100 & Platt & 0.304 & 0.102 & 0.930 \\
PPE Correctness & 400 & Labeled-only & 0.299 & 0.000 & 0.988 \\
PPE Correctness & 400 & PPI & 1.000 & 0.020 & 0.924 \\
PPE Correctness & 400 & AIPW & 0.405 & 0.009 & 0.941 \\
PPE Correctness & 400 & PPI++ & 0.276 & 0.083 & 0.989 \\
PPE Correctness & 400 & AIPW-EM & 0.262 & 0.103 & 0.989 \\
PPE Correctness & 400 & LinearCal & 0.260 & 0.105 & 0.989 \\
PPE Correctness & 400 & AutoCal & 0.263 & 0.102 & 0.988 \\
PPE Correctness & 400 & MonoSpline & 0.263 & 0.094 & 0.989 \\
PPE Correctness & 400 & IsoCal & 0.265 & 0.089 & 0.989 \\
PPE Correctness & 400 & Platt & 0.260 & 0.105 & 0.989 \\
PPE Human Preference & 100 & Labeled-only & 0.303 & 0.000 & 0.911 \\
PPE Human Preference & 100 & PPI & 1.000 & 0.004 & 0.897 \\
PPE Human Preference & 100 & AIPW & 0.919 & 0.004 & 0.896 \\
PPE Human Preference & 100 & PPI++ & 0.298 & 0.027 & 0.912 \\
PPE Human Preference & 100 & AIPW-EM & 0.295 & 0.031 & 0.909 \\
PPE Human Preference & 100 & LinearCal & 0.295 & 0.031 & 0.909 \\
PPE Human Preference & 100 & AutoCal & 0.297 & 0.028 & 0.908 \\
PPE Human Preference & 100 & MonoSpline & 0.297 & 0.025 & 0.903 \\
PPE Human Preference & 100 & IsoCal & 0.297 & 0.024 & 0.905 \\
PPE Human Preference & 400 & Labeled-only & 0.289 & 0.000 & 0.945 \\
PPE Human Preference & 400 & PPI & 1.000 & 0.005 & 0.915 \\
PPE Human Preference & 400 & AIPW & 0.703 & 0.004 & 0.917 \\
PPE Human Preference & 400 & PPI++ & 0.283 & 0.025 & 0.946 \\
PPE Human Preference & 400 & AIPW-EM & 0.282 & 0.028 & 0.945 \\
PPE Human Preference & 400 & LinearCal & 0.282 & 0.028 & 0.944 \\
PPE Human Preference & 400 & AutoCal & 0.282 & 0.027 & 0.944 \\
PPE Human Preference & 400 & MonoSpline & 0.282 & 0.027 & 0.942 \\
PPE Human Preference & 400 & IsoCal & 0.282 & 0.026 & 0.940 \\
\bottomrule
\end{tabular}
}
\caption{PPE numerical summary at \(n\in\{100,400\}\), including the efficiency-maximized and additional fixed-calibration estimators.}
\label{tab:llm-eval-summary-full}
\end{table}
}{}

\IfFileExists{assets/fig_llm_eval_by_evaluator.pdf}{
\begin{figure}[hbt]
\centering
\includegraphics[width=\textwidth]{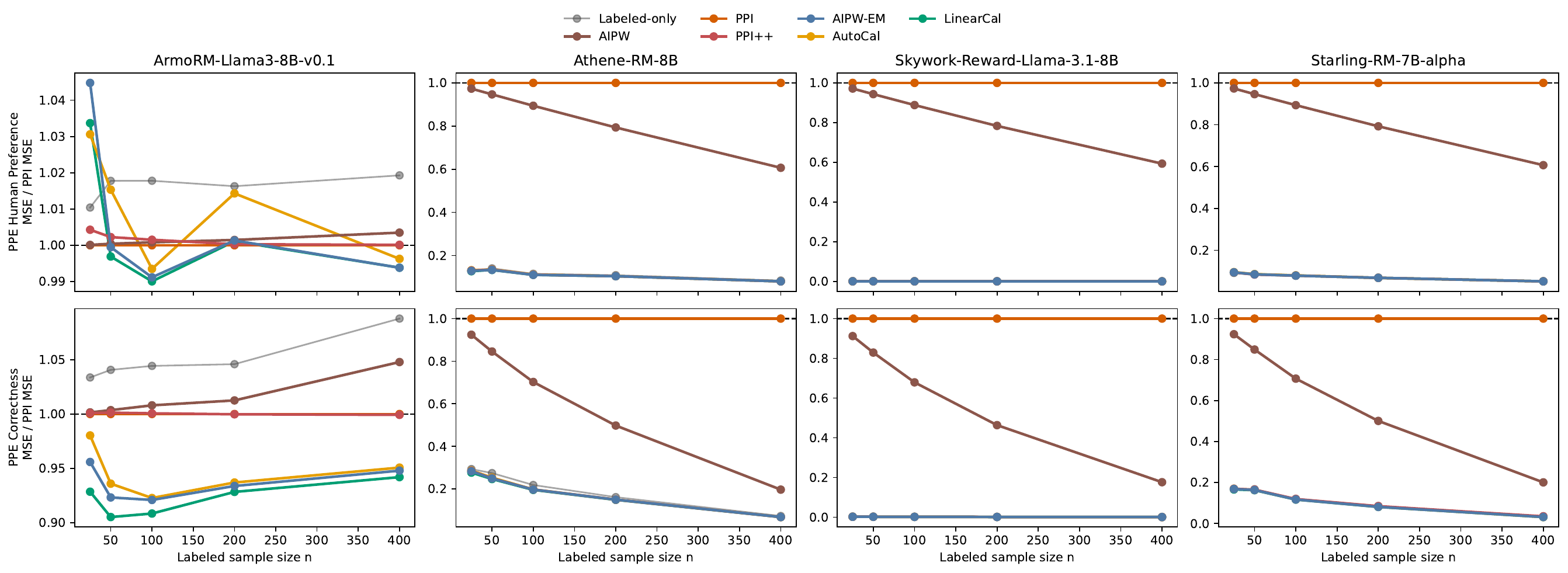}
\caption{Appendix evaluator-specific PPE summary. Each panel repeats the PPE Human and PPE Correctness comparisons for one public evaluator. The largest MSE increases relative to labeled-only estimation occur for Skywork and Athene.}
\label{fig:llm-eval-by-evaluator}
\end{figure}
}{}

\subsection{LLM ranking supplement}
\label{app:llm-ranking}

The PPE human-preference track also induces a downstream ranking problem across the selected target models. Using the same repeated labeled/unlabeled splits as in \cref{sec:llm-benchmark}, we summarize how well each estimator recovers the target-model ordering induced by the full-data human win shares. For each evaluator, labeled-sample size, replicate, and estimator, we rank the eight selected target models by estimated win share, compare that ordering with the ranking induced by the full-data targets, and then macro-average across evaluators. We report Spearman rank correlation, top-1 identification rate, and top-1 regret, defined as the difference in true win share between the best target model and the model selected by the estimator; see \cref{fig:llm-ppe-ranking,tab:llm-ppe-ranking}. The main pattern mirrors the mean-estimation results: calibrated and efficiency-maximized estimators substantially outperform raw PPI and raw AIPW, while remaining close to labeled-only at small \(n\). As shown in \cref{fig:llm-ppe-ranking}, \texttt{AutoCal} attains the highest average Spearman correlation at \(n=200\), about \(0.905\). By \(n=400\), \texttt{LinearCal} gives the best overall ranking quality, with Spearman correlation about \(0.941\) and top-1 identification rate about \(0.646\); selected numerical summaries are reported in \cref{tab:llm-ppe-ranking}.

\IfFileExists{assets/table_llm_ppe_ranking.tex}{
\begin{table}[hbt]
\centering
\small
\resizebox{0.74\textwidth}{!}{\begin{tabular}{llrrr}
\toprule
$n$ & Estimator & Spearman & Top-1 & Top-1 regret \\
\midrule
25 & AIPW & 0.268 & 0.216 & 0.064 \\
25 & AIPW-EM & 0.599 & 0.344 & 0.028 \\
25 & AutoCal & 0.600 & 0.343 & 0.028 \\
25 & Labeled-only & 0.606 & 0.315 & 0.030 \\
25 & LinearCal & 0.599 & 0.344 & 0.028 \\
25 & PPI & 0.265 & 0.216 & 0.064 \\
25 & PPI++ & 0.594 & 0.336 & 0.029 \\
100 & AIPW & 0.468 & 0.285 & 0.044 \\
100 & AIPW-EM & 0.845 & 0.477 & 0.007 \\
100 & AutoCal & 0.845 & 0.487 & 0.007 \\
100 & Labeled-only & 0.846 & 0.485 & 0.007 \\
100 & LinearCal & 0.846 & 0.482 & 0.007 \\
100 & PPI & 0.456 & 0.279 & 0.046 \\
100 & PPI++ & 0.844 & 0.479 & 0.007 \\
400 & AIPW & 0.631 & 0.439 & 0.027 \\
400 & AIPW-EM & 0.941 & 0.644 & 0.002 \\
400 & AutoCal & 0.941 & 0.640 & 0.002 \\
400 & Labeled-only & 0.940 & 0.625 & 0.002 \\
400 & LinearCal & 0.941 & 0.646 & 0.002 \\
400 & PPI & 0.579 & 0.419 & 0.031 \\
400 & PPI++ & 0.940 & 0.636 & 0.002 \\
\bottomrule
\end{tabular}
}
\caption{Appendix PPE-ranking summary at selected labeled sample sizes from the LLM benchmark. Better rankings have higher Spearman correlation and top-1 identification rates and lower top-1 regret.}
\label{tab:llm-ppe-ranking}
\end{table}
}{}

\IfFileExists{assets/fig_llm_ppe_ranking.pdf}{
\begin{figure}[hbt]
\centering
\includegraphics[width=\textwidth]{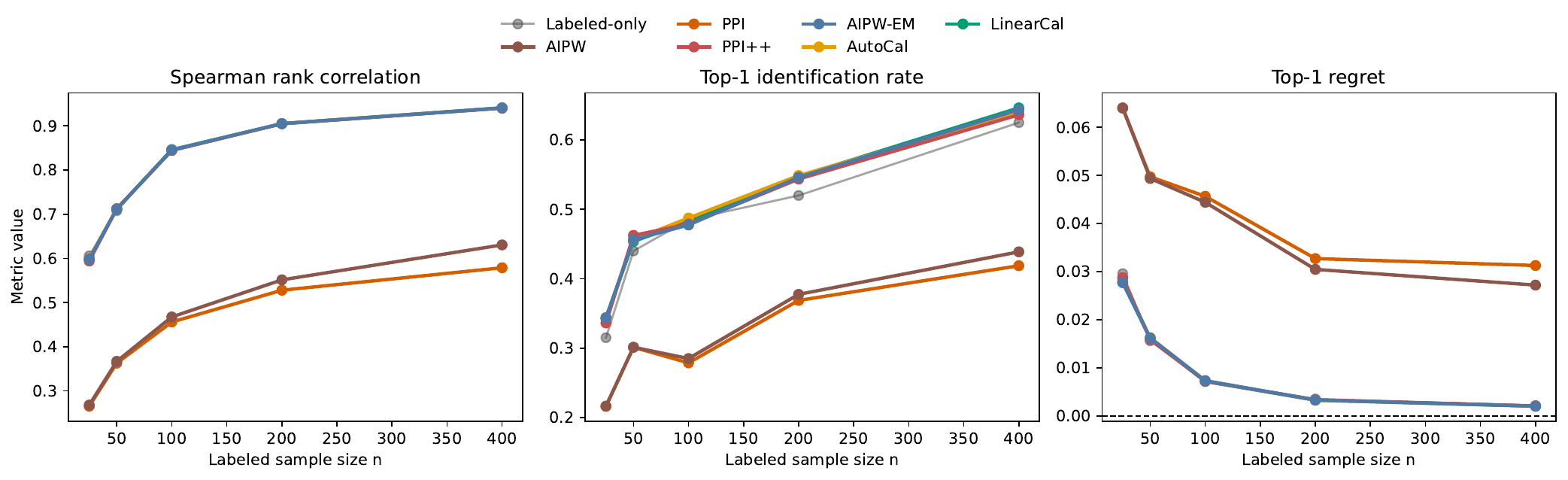}
\caption{Appendix PPE-ranking supplement from the LLM benchmark, macro-averaged across the public evaluator models. The three panels report Spearman rank correlation, top-1 identification rate, and top-1 regret for recovering the ranking of the eight selected PPE target models from a small labeled sample plus evaluator scores.}
\label{fig:llm-ppe-ranking}
\end{figure}
}{}

\section{Code availability}
\label{appendix:code}

The \texttt{ppi\_aipw} package, experiment scripts, and numerical results are
available in the \href{https://github.com/Larsvanderlaan/ppi-aipw}{project repository}.
Reference implementations are in
\href{https://github.com/Larsvanderlaan/ppi-aipw/blob/main/examples/paper_calibration.py}{\texttt{examples/paper\_calibration.py}}.
\href{https://larsvanderlaan.github.io/ppi-aipw/}{Package documentation}
includes installation instructions and worked examples.

\end{document}